\documentclass{article}
\usepackage{log_2026}						
\usepackage{multirow}
\usepackage{bigdelim}
\usepackage{amsmath}
\usepackage{amssymb}
\usepackage{amsthm}
\usepackage{dsfont}
\usepackage{booktabs}						
\usepackage{multirow}						
\usepackage{amsfonts}						
\usepackage{graphicx}						
\usepackage{duckuments}						
\usepackage{tikz}
\usetikzlibrary{positioning}
\usetikzlibrary{arrows.meta}
\usetikzlibrary{decorations.pathreplacing}
\usepackage{subcaption}
\usepackage{enumitem}
\usepackage{soul}
\usepackage{xcolor}
\usepackage{colortbl}
\definecolor{grey}{rgb}{0.5,0.5,0.5}
\newcommand{\best}[1]{%
  \sethlcolor{red!30}%
  \hl{\textbf{#1}}%
}

\newcommand{\second}[1]{%
  \sethlcolor{grey!20}%
  \hl{#1}%
}
\newcommand{\rebuttal}[1]{\textcolor{black}{#1}}
\newcommand{\newcol}[1]{\textcolor{black}{#1}}

\usepackage[
     backend=biber,
     style=numeric-comp,
     backref=true,
     natbib=true,
     url=false,
     doi=false,
     eprint=false,
     isbn=false,
     giveninits=true,
     uniquename=init]{biblatex}
\AtEveryBibitem{\clearfield{issn}\clearfield{pages}\clearlist{publisher}\clearfield{month}}

\newtheorem{remark}{Remark}
\newtheorem{theorem}{Theorem}
\newtheorem{proposition}{Proposition}

\newtheorem{corollary}{Corollary}
\title[]{Repurposing Unified Topological Signatures for Graph Representation Learning}

\author[Sanyam et al.]{
Sanyam Sanjay Jain$^{a}$,
Anshika Krishnatray$^{a}$,
Aditya Sharma$^{a}$,
Vinti Agarwal$^{a}$\\[0.5em]
$^{a}$BITS Pilani, India \quad\\[0.5em]
}

\begin{document}

\maketitle

\begin{abstract}

Message-passing Graph Neural Networks (GNNs) iteratively propagate and aggregate local neighborhood information followed by global readout to learn graph representations. However, their discriminative power is upper-bounded by the Weisfeiler--Lehman (1-WL) graph isomorphism test. This prevents GNNs from distinguishing certain non-isomorphic graphs with identical local neighborhood structures, often leading to similar graph representations. Unified Topological Signatures (UTS) capture compact, multi-scale representation of global graph topology derived from persistent homology. \rebuttal{We introduce two complementary UTS signatures: \(\Phi_{\mathrm{grph}}\), a static signature of the input graph topology, and \(\Phi_{\mathrm{emb}}\), a dynamic signature of the evolving embedding topology.} They encode structural information inaccessible to 1-WL-based message-passing GNNs, yet their capabilities are explored solely for post-hoc embedding-space analysis. We integrate UTS into GNN training across three architectural interventions: (i) \textit{UTS-Aug}: augmenting with standard readout feature that encodes graph's true topology; (ii) \textit{UTS-Reg}: topological regularizer that constrains representation collapse; (iii) \textit{UTS-Pool}: topology-guided pooling that retains structurally critical nodes.  We further leverage UTS as a layer-wise diagnostic to quantify oversmoothing during GNN training.
Theoretically, we show that integrating UTS into GNN optimization strictly extends GNN expressivity beyond the 1-WL hierarchy. 
\rebuttal{Experiments on three graph classification benchmarks show consistent benefits: Graph-UTS, Dual-UTS, and UTS-Pool improve accuracy across all three datasets, Embedding-UTS provides smaller but similarly consistent gains, and UTS-Reg's benefit varies across graph domains. Accuracy improves by up to $5.8\%$ with Graph-UTS augmentation, by up to $1.9\%$ with UTS-Reg, and achieves comparable performance to TOGL with UTS-Pool.}

\end{abstract}

\section{Introduction}
\label{sec:intro}


Graph Neural Networks (GNNs) have become the default architecture for relational learning, propagating and aggregating information across neighborhoods in a permutation-invariant manner~\citep{hamilton2018inductiverepresentationlearninglarge}.
For graph-level tasks, this aggregation exposes a representational limitation at readout. Standard readout functions (sum, mean, max) \ul{retain only first-order statistics of node embeddings}, discarding higher-order connectivity almost entirely. Learned pooling operators perform adaptive graph coarsening~\citep{Grattarola_2024,
bianchi2020spectralclusteringgraphneural,
ying2019hierarchicalgraphrepresentationlearning},  
but select nodes based on local neighborhood, offering no guarantee that global topology survives pooling. 
Deeper message-passing layers enlarge the receptive field, but repeated aggregation acts as a low-pass filter on the feature matrix, \ul{collapsing node embeddings toward a common value} as the Dirichlet energy decays exponentially in depth---resulting in
\textit{oversmoothing}~\citep{li2018deeper, oono2020graph}. Oversmoothing collapses node embeddings and, consequently, graph representations after readout, eliminating structural information essential for graph-level prediction. Standard GNN objectives also lack an explicit mechanism to monitor topological fidelity during training.

The expressive power of message-passing GNNs is provably \ul{upper bounded by the $1$-dimensional Weisfeiler--Leman ($1$-WL) test}: any two graphs that are $1$-WL equivalent yield identical node embeddings under any message-passing GNN, regardless of depth or width~\citep{xu2019how}. As a result, non-isomorphic graphs with distinct global organization (\ e.g., G1 and G2 in Figure~\ref{fig:uts_comparison}), collapse to indistinguishable representations despite having different cycle structure and connectivity. 

Topological Data Analysis offers a natural way to expose what message passing hides. Persistent homology \citep{Edelsbrunner2002TopologicalPersistence}  extracts stable, multi-scale descriptors (components, cycles, higher-order voids), while recent work has incorporated topological priors into GNNs through persistence-based pooling~\citep{chen2023topologicalpoolinggraphs} and discrete curvature methods such as Ricci flow and ORC-POOL~\citep{ollivier2007riccicurvaturemarkovchains, feng2024graphpoolingricciflow}. These confirm topological signal benefits GNNs, but each draws on a single geometric perspective. 
\emph{Unified Topological Signatures} (UTS) compresses the global geometry of an embedding space into a compact descriptor combining persistent homology with geometric and spectral features~\citep{rottach2025topologyretrievaldecodingembedding}, but so far only as a post-hoc analysis tool after training.

We incorporate UTS into the GNN optimization process through two complementary signatures: a \ul{dynamic embedding signature} (\textit{Embedding-UTS}) computed from each GNN layer, and a \ul{static graph topology signature} (\textit{Graph-UTS}) extracted from the input graph. The former characterizes the evolving topology of the learned representation, while the latter captures the graph's intrinsic topological invariants. 
For example, Figure~\ref{fig:uts_comparison} illustrates that the 1-WL-equivalent graphs $G_1$ ($C_6$) and $G_2$ ($2\times C_3$) are indistinguishable to message-passing GNNs, whereas UTS captures complementary global topological information beyond neighborhood aggregation, yielding distinct graph representations.

We deploy both signatures across three architectural interventions: augmented readout, topology-preserving regularization, and topology-aware pooling, detailed in Section~\ref{sec:methodology}. The dynamic signature additionally provides a layer-wise \textit{Oversmoothing Index}, transforming representation collapse from an unobserved failure mode into a measurable training signal.

\begin{figure}[t]
    \centering
    \includegraphics[width=1\linewidth]{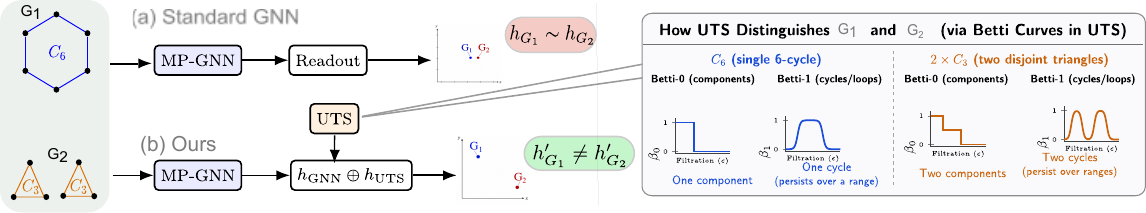}
    \caption{Illustration of UTS-based graph discrimination. Standard message-passing GNNs with global readout may map the topologically distinct graphs ${G_1}(C_6) $ and ${G_2} (2\times C_3) $ to similar embeddings due to their reliance on local neighborhood aggregation. {UTS} resolves this limitation by integrating a comprehensive topological signature derived from persistent homology. For clarity, only the \textit{UTS-Aug} with Betti-0 and Betti-1 curves are visualized, illustrating how differences in connected components and cycles contribute to distinguishing the two graphs.}
    \label{fig:uts_comparison}
\end{figure}

We instantiate our framework on Graph Isomorphism Networks (GINs)~\citep{xu2019how}, whose expressiveness matches the 1-WL hierarchy, thereby isolating the contribution of UTS from architectural capacity. We theoretically establish that UTS-augmented GINs strictly exceed the expressive power of standard GINs while improving topological fidelity and robustness to oversmoothing. Experimental evaluation on three graph classification benchmarks confirms consistent gains over competitive message-passing baselines.

\vspace{0.3em}
\noindent\textbf{Contributions.} 
Our main contributions are summarized as follows:
\begin{itemize}
    \item We introduce a dual topological signature: a layer-wise \emph{embedding} signature (Embedding-UTS), $\Phi_{\mathrm{emb}}$ tracking the evolving geometry of learned representations, and a static \emph{graph} signature (Graph-UTS) $\Phi_{\mathrm{grph}}$ characterizing input topology independently of training and immune to oversmoothing.
    \item We integrate these signatures into GNN training as auxiliary features, regularization signals, and pooling guidance and perform experiments on three graph classification benchmarks (MUTAG, PROTEINS, and COLLAB).
    \item We prove the UTS-augmented readout is strictly more expressive than the 1-WL test (Theorem~\ref{thm:expressivity}).
    \item We introduce the \textit{Oversmoothing Index (OSI)}, a graph-level metric quantifying topological degradation across layers during training, without sacrificing differentiability or training efficiency.
\end{itemize}
\section{Related work}
\label{sec:related}

Graph Neural Networks combine node features with connectivity through neighborhood aggregation. GraphSAGE introduced inductive message passing via learned sampled aggregators~\citep{hamilton2018inductiverepresentationlearninglarge}, while GIN established sum aggregation with expressive MLPs as the most discriminative scheme under the Weisfeiler-Lehman framework~\citep{xu2019how}.

Graph classification additionally requires graph-level readout. Simple sum/mean/max pooling compress graphs into first-order statistics; hierarchical methods such as DiffPool and MinCutPool instead learn differentiable cluster assignments~\citep{ying2019hierarchicalgraphrepresentationlearning, bianchi2020spectralclusteringgraphneural}, Select-Reduce-Connect unifies pooling methods by how they select, reduce, and reconnect nodes~\citep{Grattarola_2024}, and Graph Reference Distribution Learning represents graphs as distributions over node embeddings rather than a single pooled vector~\citep{wang2024graphclassificationreferencedistribution}. In contrast, our UTS-guided pooling ranks nodes by the topological richness of their embedding neighborhoods.


A separate line embeds topology directly into GNNs: TOGL uses persistent homology as a differentiable, 1-WL-exceeding layer~\citep{horn2022topologicalgraphneuralnetworks}; Wit-TopoPool applies witness-complex persistence to hierarchical pooling~\citep{chen2023topologicalpoolinggraphs}. We instead use Unified Topological Signatures (UTS)~\citep{rottach2025topologyretrievaldecodingembedding}, combining persistence with geometric/spectral statistics into a compact descriptor for post-hoc analysis. Ollivier-Ricci curvature characterizes community structure and bridge edges~\citep{ollivier2007riccicurvaturemarkovchains}, exploited by ORC-POOL for pooling~\citep{feng2024graphpoolingricciflow}; we fold curvature into UTS instead. Oversmoothing diagnostics typically rely on post-training Dirichlet energy~\citep{rusch2023surveyoversmoothing}, missing topological preservation; our Oversmoothing Index signals topology throughout training. Prior work thus couples topology to specialized mechanisms (TOGL, Wit-TopoPool, ORC-POOL) or post-hoc analysis (UTS, Dirichlet diagnostics); we repurpose UTS as a unified training-time signal for readout, regularization, and pooling.
\section{Preliminaries}
\label{sec:preliminaries}

\subsection{Problem Setting}

Let $\mathcal{D}=\{(G_i,y_i)\}_{i=1}^{N}$ denote a graph classification dataset, where each graph
$G_i=(\mathcal{V}_i,\mathcal{E}_i,\mathbf{X}_i)$ consists of a node set
$\mathcal{V}_i$, an edge set
$\mathcal{E}_i$, and node feature matrix
$\mathbf{X}_i\in\mathbb{R}^{|\mathcal{V}_i|\times d_0}$.
The objective is to learn a permutation-invariant function $f:\mathcal{G}\rightarrow\mathcal{Y},$ which maps an input graph to its class label while preserving the structural information required for graph-level discrimination. Throughout this work, we investigate how explicit topological descriptors can complement conventional message passing to improve graph representations for graph classification.

\subsection{GIN Backbone}
\label{sec:gin}

We instantiate our framework on the Graph Isomorphism Network (GIN)~\citep{xu2019how}, chosen because its sum-based aggregation matches the expressive power of the $1$-Weisfeiler--Leman (1-WL) test among message-passing GNNs. Given node representations $\mathbf{H}^{(l-1)}$, each layer updates node $v$ as

\begin{equation}
\mathbf{h}_v^{(l)}
=
\mathrm{MLP}^{(l)}
\left(
(1+\epsilon^{(l)})\mathbf{h}_v^{(l-1)}
+
\sum_{u\in\mathcal N(v)}
\mathbf{h}_u^{(l-1)}
\right),
\label{eq:gin}
\end{equation}

where $\epsilon^{(l)}$ is a learnable scalar and $\mathbf{h}_v^{(0)}=\mathbf{x}_v$. A permutation-invariant graph readout aggregates the final node embeddings into a graph representation. Our proposed framework augments this standard pipeline with explicit topological information while leaving the underlying message-passing architecture unchanged.

\section{Methodology}
\label{sec:methodology}

We introduce \emph{Dual Unified Topological Signatures}, a pair of complementary descriptors that explicitly incorporate topology into graph learning: \textbf{(i)} a dynamic {embedding} signature (\textit{Embedding-UTS}) $\Phi_{\mathrm{emb}}$ captures the evolving topology of the learned representation throughout message passing, and \textbf{(ii)} a static {graph} signature (\textit{Graph-UTS}),  $\Phi_{\mathrm{grph}}$ characterizes the intrinsic topology of the input graph. 
These signatures support three independent interventions:
\begin{enumerate}[label=(\roman*), leftmargin=*, itemsep=1pt, topsep=2pt]
    \item \textbf{\textbf{UTS-Aug:} Topology-Augmented Graph Representation} (\S\ref{subsec:descriptor}) concatenates the graph and embedding signatures with the conventional graph readout.

    \item \textbf{\textbf{UTS-Reg:}Topology-Preserving Regularization} (\S\ref{subsec:regularization}) introduces auxiliary objectives that preserve the evolution of embedding topology during training while aligning the final embedding topology with the intrinsic graph topology.

    \item \textbf{\textbf{UTS-Pool:}Topology-Aware Pooling} (\S\ref{subsec:pooling}) replaces feature-based node ranking with a topology-guided scoring mechanism derived from local embedding neighborhoods.
\end{enumerate}

Figure~\ref{fig:full_architecture} presents an overview of three approaches. Each intervention is modular and can be incorporated independently or in combination, without modifying the underlying GNN architecture. The evolution of the embedding signature additionally provides a graph-level oversmoothing diagnostic. 

\subsection{Dual Unified Topological Signatures}
\label{subsec:dual_uts}
Conventional graph readout aggregates node embeddings but lacks an explicit mechanism to encode global topology beyond message-passing representations.
To address this limitation, we introduce two complementary topological signatures that characterize the different topological perspectives of input graph: 
Graph-UTS characterizes \ul{``What topological structure is present in input graph''}, while Embedding-UTS tells \ul{``How that topological structure is reflected in the learned representations''}.
\subsubsection{Embedding-UTS}
\label{subsec:embed_uts}

Given the node embedding matrix
$\mathbf{H}^{(l)}\in\mathbb{R}^{n\times d_l}$ at layer $l$ in GNN, we interpret its rows as a point cloud in the learned representation space and compute an embedding signature $\Phi_{\mathrm{emb}}: \mathbb{R}^{n \times d} \rightarrow \mathbb{R}^{14}$
which summarizes the geometry of the embedding through geometric, persistent homology, and spectral descriptors. \rebuttal{We use two variants of this signature depending on whether gradients must propagate through it: an \emph{exact} variant, computed via standard persistent homology and eigendecomposition, used where the signature serves only as a static input feature (UTS-Aug, \S\ref{subsec:descriptor}); and a \emph{differentiable surrogate}, denoted $\Phi_{\mathrm{emb}}^{\mathrm{diff}}$, which replaces these exact operations with differentiable relaxations and is used wherever the signature must be trained through (UTS-Reg, \S\ref{subsec:regularization}; UTS-Pool, \S\ref{subsec:pooling}). Both variants share the feature structure below.}
The resulting embedding descriptor is 
{\small
\begin{equation}
    \mathbf{s}^{(l)} = \Phi_{\mathrm{emb}}\!\left(\mathbf{H}^{(l)}\right)
    = \big[\,
        \underbrace{\bar{\ell}^{(0)},\, H^{(0)},\,
        \bar{\ell}^{(1)},\, H^{(1)},\,
        \beta_0,\, \beta_1}_{\text{persistence}},\,
        \underbrace{\mu_{\mathrm{nn}},\, \sigma_{\mathrm{nn}},\,
        \Delta,\, \hat{d}}_{\text{local geometry}},\,
        \underbrace{\lambda_1,\, \lambda_2,\, \lambda_3,\,
        H_{\mathrm{spec}}}_{\text{spectral}}
    \,\big] \in \mathbb{R}^{14}.
    \label{eq:embed_uts_vec}
\end{equation}}
Here $\bar{\ell}^{(k)}$, $H^{(k)}$, and $\beta_k$ denote the mean
lifetime, persistence entropy, and Betti number of $k$-dimensional
homology; $\mu_{\mathrm{nn}}$ and $\sigma_{\mathrm{nn}}$ are the mean and
standard deviation of $k$-NN distances; $\Delta$ is the diameter of point
cloud $\mathcal{P}$; $\hat{d}$ estimates its intrinsic dimensionality; and
$\lambda_1 \leq \lambda_2 \leq \lambda_3$ with $H_{\mathrm{spec}}$ are the
smallest eigenvalues and spectral entropy of a Laplacian. \rebuttal{For $\Phi_{\mathrm{emb}}$, $\beta_k$ and the Laplacian are computed exactly; for $\Phi_{\mathrm{emb}}^{\mathrm{diff}}$, both are differentiable relaxations (soft Betti numbers, a regularized Laplacian). Complete derivations of both variants are provided in Appendix~\ref{app:embed_uts_full}.}

\subsubsection{Graph-UTS}
\label{subsec:graph_uts}

Complementary to the embedding signature, we compute a static graph signature directly from the input topology, $\Phi_{\mathrm{grph}}: \mathcal{G} \rightarrow \mathbb{R}^{27}$ which is evaluated once for every graph and cached throughout training. 
The graph signature combines curvature, persistent homology, spectral, distance, and structural statistics into a compact topological representation, 
{\small
\begin{equation}
    \mathbf{g} = \Phi_{\mathrm{grph}}(G)
    = \big[\;
        \underbrace{\phi_{\mathrm{ORC}},\,
        \phi_{\mathrm{FRC}}}_{\text{curvature}},\,
        \underbrace{\bar\ell_G,\, \mathrm{diam}(G)}_{\text{distance}},\,
        \underbrace{\lambda_2^G,\, \lambda_n^G,\,
        H_{\mathrm{spec}}^G}_{\text{spectral}},\,
        \underbrace{\bar\ell^{(0)}_G,\, H^{(0)}_G,\,
        \bar\ell^{(1)}_G,\, H^{(1)}_G}_{\text{persistence}},\,
        \underbrace{\phi_{\mathrm{struct}}}_{\text{structural}}
    \;\big] \in \mathbb{R}^{27}.
    \label{eq:graph_uts_vec}
\end{equation}}
Here $\phi_{\mathrm{ORC}}, \phi_{\mathrm{FRC}} \in \mathbb{R}^4$ summarize
Ollivier--Ricci and Forman--Ricci edge curvatures; $\bar\ell_G$ and
$\mathrm{diam}(G)$ are the mean shortest-path length and diameter of graph;
$\lambda_2^G$, $\lambda_n^G$, and $H_{\mathrm{spec}}^G$ are the algebraic
connectivity, largest eigenvalue, and spectral entropy of the combinatorial
Laplacian; $\bar\ell^{(k)}_G$ and $H^{(k)}_G$ are exact $k$-dimensional
persistence lifetimes and entropies; and $\phi_{\mathrm{struct}} \in
\mathbb{R}^{10}$ collects elementary structural statistics including betweenness centrality, clustering coefficients, and more.
Complete feature definitions are provided in Appendix~\ref{app:graph_uts_full}. Unlike $\Phi_{\mathrm{emb}}$, this descriptor is independent of the learned embeddings and therefore remains unaffected by message passing or oversmoothing.
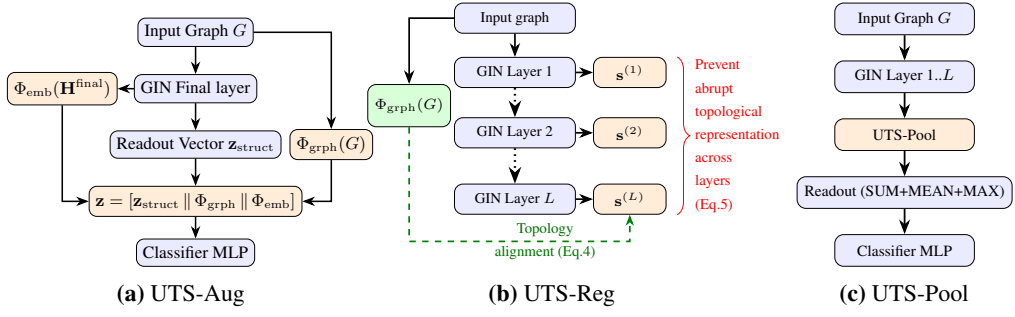
\begin{figure*}[t]
\centering
\begin{subfigure}[b]{0.33\linewidth}
\centering
\resizebox{!}{3.3cm}{%
\begin{tikzpicture}[
  box/.style={draw, rounded corners, fill=blue!8, minimum width=2.0cm, minimum height=0.5cm, font=\small, align=center},
  sbox/.style={draw, rounded corners, fill=orange!15, minimum width=1.5cm, minimum height=0.5cm, font=\small, align=center},
  arr/.style={-{Stealth}, thick}, node distance=0.45cm
]
  \node[box] (input) {Input Graph $G$};
  \node[box,below=0.5cm of input] (Gin)   {GIN Final layer};
  \node[sbox, left=0.3cm of Gin] (embuts) {$\Phi_{\mathrm{emb}}(\mathbf{H}^{\mathrm{final}})$};
  \node[box, below=0.5cm of Gin] (readout) {Readout Vector $\mathbf{z}_{\mathrm{struct}}$};
  \node[sbox, right=0.15cm of readout] (guts) {$\Phi_{\mathrm{grph}}(G)$};
  \node[sbox, below=0.5cm of readout] (concat) {$\mathbf{z} = [\mathbf{z}_{\mathrm{struct}} \,\|\, \Phi_{\mathrm{grph}} \,\|\, \Phi_{\mathrm{emb}}]$};
  \node[box, below=0.4cm of concat] (cls) {Classifier MLP};

  \draw[arr] (input) -- (Gin);
  \draw[arr] (Gin) -- (readout);
  \draw[arr] (readout) -- (concat);
  \draw[arr] (concat) -- (cls);

  \draw[arr] (Gin.west) -- (embuts.east);
  \draw[arr] (embuts.south) |- (concat.west);

  \draw[arr] (input.east) -| (guts.north);
  \draw[arr] (guts.south) |- (concat.east);
\end{tikzpicture}%
}
\caption{UTS-Aug}
\end{subfigure}
\hfill \hfill
\begin{subfigure}[b]{0.35\linewidth}
\resizebox{!}{3.5cm}{%
  

\begin{tikzpicture}[
  box/.style={draw, rounded corners, fill=blue!8, minimum width=2.0cm,
    minimum height=0.5cm, font=\scriptsize, align=center},
  sbox/.style={draw, rounded corners, fill=orange!15, minimum width=1.2cm,
    minimum height=0.5cm, font=\scriptsize, align=center},
  gbox/.style={draw, rounded corners, fill=green!15, minimum width=1.2cm,
    minimum height=0.7cm, font=\scriptsize, align=center},
  arr/.style={-{Stealth}, thick}, node distance=0.4cm
]
  \node[box] (input) at (0, 3.0)  {Input graph};
  \node[box] (gin1)  at (0, 2.1) {GIN Layer 1};
  \node[box] (gin2)  at (0, 1.1) {GIN Layer 2};
  \node[box] (ginL)  at (0, 0.0) {GIN Layer $L$};

  \node[sbox] (s1) at (1.9, 2.1) {$\mathbf{s}^{(1)}$};
  \node[sbox] (s2) at (1.9, 1.1) {$\mathbf{s}^{(2)}$};
  \node[sbox] (sL) at (1.9, 0.0) {$\mathbf{s}^{(L)}$};

  \node[gbox] (guts) at (-1.75, 1.55) {$\Phi_{\mathrm{grph}}(G)$};

  \draw[arr] (input) -- (gin1);
  \draw[arr, dotted, line width=1.1pt] (gin1) -- (gin2);
  \draw[arr, dotted, line width=1.1pt] (gin2) -- (ginL);

  \draw[arr] (gin1) -- (s1);
  \draw[arr] (gin2) -- (s2);
  \draw[arr] (ginL) -- (sL);

  \draw[arr] (input.west) -| (guts.north);

  \coordinate (midpt) at (1.9, -0.7);
  \draw[green!50!black, dashed, thick, -{Stealth[scale=0.8]}]
    (guts.south) |- (midpt) -- (sL.south);
  \node[green!45!black, font=\normalsize, align=center] at (0.5, -0.7)
    {\scriptsize Topology\\\scriptsize alignment (Eq.4)};

  \draw[red, decorate, decoration={brace, amplitude=6pt, raise=5pt}]
    (s1.north east) -- (sL.south east)
    node[midway, right=10pt, font=\normalsize, align=left, red]
    {\scriptsize Prevent\\\scriptsize abrupt\\\scriptsize topological\\\scriptsize representation\\\scriptsize across\\\scriptsize layers\\ \scriptsize (Eq.5)};
\end{tikzpicture}
}
\caption{UTS-Reg}
\end{subfigure}
\hfill
\begin{subfigure}[b]{0.3\linewidth}
\centering
\resizebox{!}{3.5cm}{%
\begin{tikzpicture}[
  box/.style={draw, rounded corners, fill=blue!8, minimum width=2.2cm, minimum height=0.5cm, font=\scriptsize, align=center},
  sbox/.style={draw, rounded corners, fill=orange!15, minimum width=2.2cm, minimum height=0.5cm, font=\scriptsize, align=center},
  arr/.style={-{Stealth}, thick}, node distance=0.4cm
]
  \node[box] (input) {Input Graph $G$};
  \node[box, below=of input] (gin1) {GIN Layer 1..$L$};
  \node[sbox, below=of gin1] (pool) {UTS-Pool};
  \node[box, below=of pool] (readout) {Readout (SUM+MEAN+MAX)};
  \node[box, below=of readout] (cls) {Classifier MLP};
  \draw[arr] (input) -- (gin1); \draw[arr] (gin1) -- (pool);
  \draw[arr] (pool) -- (readout); \draw[arr] (readout) -- (cls);
\end{tikzpicture}%
}
\caption{UTS-Pool}
\end{subfigure}
\caption{Overview of Dual-UTS interventions applied on the baseline GIN architecture.}
\label{fig:full_architecture}
\end{figure*}
\subsection{UTS-Aug: Topology-Enhanced Graph Representation}
\label{subsec:descriptor}


We augment the graph representation obtained from conventional readout with the topological signatures introduced in Section~\ref{subsec:graph_uts}. Graph-UTS $\Phi_{\mathrm{grph}}$ encodes the intrinsic topology of the input graph and remains invariant throughout optimization, whereas Embedding-UTS $\Phi_{\mathrm{emb}}$ captures the evolving topological information of intermediate representations during message passing.

Let $\mathbf{z}_{\mathrm{struct}} \in \mathbb{R}^{d}$ denote the graph representation obtained from the underlying GNN through a standard permutation-invariant readout. The proposed topology-aware representation is constructed by concatenating the graph signature and embedding signature of last layer, 
\[
\mathbf{z}
=
\left[
\mathbf{z}_{\mathrm{struct}}
\;\|\;
\Phi_{\mathrm{grph}}(G) 
\;\|\;
\Phi_{\mathrm{emb}}(\textbf{H}^{(l)}) 
\right],
\label{eq:augmented_readout}
\]
 where $\|$ denotes vector concatenation. The augmented representation $\mathbf{z}$ is subsequently passed to the graph classifier. To isolate the contribution of each signature, we additionally evaluate single-signature variants by concatenating $\mathbf{z}_{\mathrm{struct}}$ with either $\Phi_{\mathrm{grph}} $ or $\Phi_{\mathrm{emb}}$ alone in section~\ref{subsec:descriptor_results}.



\subsection{UTS-Reg: Topology-Preserving Regularization}
\label{subsec:regularization}


While \textbf{UTS-Aug} augments the final graph representation with explicit topological information, it does not constrain the evolution of topology during message passing. \textbf{UTS-Reg} addresses this through two regularization objectives: (i) aligning the topology of final embedding space with the intrinsic topology of input graph, and (ii) enforcing topological consistency across successive GNN layers. Together, these objectives preserve structural fidelity, reduce representation collapse and thus mitigate oversmoothing throughout optimization.

\paragraph{Topological Alignment Loss}

Since, the graph signature $\Phi_{\mathrm{grph}}(G)$ and embedding signature \rebuttal{$\Phi_{\mathrm{emb}}^{\mathrm{diff}}(\mathbf{H}^{(L)})$} reside in different feature spaces, we learn a linear projection $\mathbf{W}_{\mathrm{align}}\in\mathbb{R}^{14\times27}$ to map the graph-signature into the embedding-signature space and minimize:
\begin{equation}
\mathcal{L}_{\mathrm{topo-align}}
=
\frac{1}{N}
\sum_{i=1}^{N}
\left\|
\Phi_{\mathrm{emb}}^{\mathrm{diff}}(\mathbf{H}_i^{(L)})
-
\mathbf{W}_{\mathrm{align}}
\Phi_{\mathrm{grph}}(G_i)
\right\|_2^2,
\label{eq:topo_reg}
\end{equation}
where $\mathbf{W}_{\mathrm{align}}$ is jointly optimized with the GNN. Since $\Phi_{\mathrm{grph}}$ is fixed, gradients propagate only through \rebuttal{$\Phi_{\mathrm{emb}}^{\mathrm{diff}}$} and $\mathbf{W}_{\mathrm{align}}$, encouraging the learned embedding topology to align with the intrinsic topology of the input graph.
\paragraph{Topological Evolution Loss}


\rebuttal{We additionally penalize abrupt topological changes across successive message-passing layers. Let $ \mathbf{s}^{(l)}=\Phi_{\mathrm{emb}}^{\mathrm{diff}}\!\left(\mathbf{H}^{(l)}\right)
$ denote the embedding signature at layer $l$.} The layer-wise smoothness objective is
\begin{equation}
\mathcal{L}_{\mathrm{topo-evol}}
=
\frac{1}{L-1}
\sum_{l=1}^{L-1}
\left\|
\mathbf{s}^{(l)}
-
\mathbf{s}^{(l-1)}
\right\|_2^2,
\label{eq:layer_smooth}
\end{equation}

which acts as a \textit{fitting} constraint and encourages smooth topological transitions across layers i.e., do not change the topology too much from the initial one, preventing representation collapse and oversmoothing.


\textbf{Total Training Objective:}
The final optimization objective combines the task loss with both topology-preserving regularizers,

\begin{equation}
\mathcal{L}
=
\mathcal{L}_{\mathrm{task}}
+
\mathcal{L}_{\mathrm{topo-align}}
+
\mathcal{L}_{\mathrm{topo-evol}},
\label{eq:total_loss}
\end{equation}


\subsection{UTS-Pool: Topology-aware pooling}
\label{subsec:pooling}

Existing graph pooling methods~\citep{Grattarola_2024,
bianchi2020spectralclusteringgraphneural,
ying2019hierarchicalgraphrepresentationlearning} rank nodes using feature activations or attention scores, emphasizing local semantic relevance without explicitly accounting for local topological structure. In our work, we rank 
nodes based on the topological richness of their embedding neighborhoods, utilizing a differentiable version of $\Phi_{\mathrm{emb}}$ which preserves structurally critical regions during graph coarsening. In addition, we also propose a lightweight geometric approximation of node scoring on large dense graphs to improve scalability.
\paragraph{Topology-Aware Node Scoring}
For each node $v$, let $\mathbf{H}_{\mathcal{N}(v)} = \{\mathbf{h}_u: u \in \mathcal{N}(v) \cup \{v\}\}$ denote the embeddings within its one-hop neighborhood. We compute a local topological signature and transform it into an importance score, 
\[\alpha_v
=
\sigma
\!\left(
f_\theta
\!\left(
\Phi_{\mathrm{emb}}^{\mathrm{diff}}
(
\mathbf H_{\mathcal N(v)}
)
\right)
\right),
\label{eq:uts_score}
\]
where $f_\theta:\mathbb R^{14}\rightarrow\mathbb R$ is a two-layer MLP, $\sigma(.)$ is non-linear function and $\Phi_{\mathrm{emb}}^{\mathrm{diff}}$
denotes differentiable embedding signature (details in Appendix~\ref{app:diff_uts}).
The top-k=
$\lfloor\rho n\rfloor$ out of $n$
nodes are retained, 
$
\mathcal V'
=
\operatorname{top\text{-}k}
(
\{\alpha_v\}_{v\in\mathcal V}),
$ 
where $\rho$ denotes the pooling ratio. The pooled graph
$G'=(\mathcal V',\mathcal E',\mathbf H')$
is obtained by preserving edges between retained nodes and projecting the corresponding node embeddings,

\begin{equation}
\mathbf H'
=
\operatorname{ReLU}
(
\mathbf W_{\mathrm{proj}}
\mathbf H[\mathcal V']
).
\end{equation}

\paragraph{Efficient Variant for Large Graphs}

Computing the embedding signature in large dense graph is  becomes computationally expensive owing to large scale persistent homology, triangle enumeration, and spectral computations for every node neighborhood. To improve scalability, we introduce a lightweight geometric approximation that replaces the full signature with five differentiable local descriptors,  

\begin{equation}
    \phi_v^{\mathrm{light}} = \left[
        \mu_{\mathrm{nn}}^v,\;
        \sigma_{\mathrm{nn}}^v,\;
        \Delta^v,\;
        \bar{r}_1^v,\;
        |\mathcal{N}(v)|/k_{\max}
    \right] \in \mathbb{R}^5,
    \label{eq:light_scorer}
\end{equation}

where each quantity is computed within the local neighbourhood of $v$. The resulting scorer reduces the computational cost from cubic topological computations to local quadratic operations while preserving the geometric characteristics most relevant for node ranking. During training, neighbourhoods are processed in chunks of size $C$ (default $C=64$), bounding peak GPU memory usage to $O(Cd)$ independent of graph size.

The lightweight scorer $\phi_v^{\mathrm{light}}$ (Eq.~\ref{eq:light_scorer}) is a
geometric proxy for the full embedding signature $\Phi_{\mathrm{emb}}^{\mathrm{diff}}$
We treat this approximation rigorously in
Appendix~\ref{app:proofs}, where we show that $\phi_v^{\mathrm{light}}$ certifies an
explicit scale window, noise floor, and dimension cap governing when a node's local
neighbourhood can carry nontrivial persistent topology at all — i.e.\ the conditions
under which the full topological computation would be uninformative and the
lightweight proxy suffices.

\section{Experiments and Results}
\label{sec:experiments}

\subsection{Experimental Protocol}

We evaluate the proposed modified GIN framework using three UTS interventions (\textit{UTS-Aug, UTS-Reg, UTS-Pool}) and perform experiments on four graph classification benchmarks spanning molecular (MUTAG), biological (PROTEINS), social-network (COLLAB), and large-scale biological (\rebuttal{ogbg-ppa}) domains. Details are provided in Appendix~\ref{app:datasets}.

In Table~\ref{tab:descriptor} and Table~\ref{tab:loss}, the proposed UTS-Aug and UTS-Reg interventions are compared against the unmodified GIN (GIN with only standard readout features, $\mathbf{z}_{\mathrm{struct}}$) and unregularized GIN (GIN trained with task-specific objective, $\mathcal{L}_{\mathrm{task}}$ only) variants respectively. In Table~\ref{tab:pooling}, modified GIN with UTS-Pooling is compared with three pooling baselines, TopKPool \citep{cangea2018sparsehierarchicalgraphclassifiers}, SAGPool \citep{pmlr-v97-lee19c}, and TOGL \citep{horn2022topologicalgraphneuralnetworks}. Across all four benchmarks, the GIN backbone is trained under identical experimental settings (hidden size 128, readout dimension 384) to isolate the contribution of the proposed topology-aware components.

\rebuttal{For MUTAG, PROTEINS, and COLLAB, which lack a standardized split, we employ 10-fold stratified cross-validation, using 9 seeds for MUTAG and 5 seeds for PROTEINS and COLLAB. Each fold uses a 90/10 train-test split, with 10\% of the training set further held out for validation. For \textbf{ogbg-ppa}, we instead use the official, externally-fixed OGB species split~\citep{hu2020ogb} and report mean $\pm$ std over multiple seeds, following standard OGB leaderboard convention; this split holds out entire species unseen during training, testing out-of-distribution generalization rather than in-distribution accuracy. Full aggregate statistics, including confidence intervals and paired-significance $p$-values for all four datasets, are provided in Appendix~\ref{app:ablation}.}
The GIN backbone uses four layers for MUTAG and PROTEINS, three layers for COLLAB to mitigate oversmoothing in dense social graphs, \rebuttal{and three layers for ogbg-ppa given its large average graph size (243.4 nodes)}. We report the mean and standard deviation of the classification accuracy (mean $\pm$ std) across all folds and seeds (TU datasets) or all seeds (ogbg-ppa). To ensure a fair comparison, all GIN variants are trained using identical hyperparameters and optimization settings, with only the proposed topology-aware components varying across the experiments.

\subsection{Effect of Dual Unified Topological Signatures as UTS-Aug}
\label{subsec:descriptor_results}

\rebuttal{Table~\ref{tab:descriptor} summarizes the classification performance of GIN augmented with each signature individually and jointly, under the updated 10-fold CV protocol (MUTAG/PROTEINS/COLLAB) and the official OGB split (ogbg-ppa). Graph-UTS achieves the best performance on MUTAG ($+5.8\%$ over unmodified GIN, not statistically significant at this sample size), while Dual-UTS achieves the best performance on PROTEINS ($+2.3\%$, significant), COLLAB ($+2.6\%$, significant), \newcol{and ogbg-ppa ($+3.22$ points, significant)}, with Graph-UTS second-best on all three. The two signatures are complementary on PROTEINS, COLLAB, \newcol{and ogbg-ppa}, with their combination outperforming either individually, while Graph-UTS alone remains strongest specifically on the smaller, sparser MUTAG graphs.}

\begin{table}[t]
\centering
\setlength{\tabcolsep}{1.5pt}
\caption{ Accuracy comparison of \textit{UTS-Aug} GIN backbone with $\Phi_{\mathrm{emb}}$, $\Phi_{\mathrm{grph}}$, or both (Dual UTS), with the unmodified GIN. \best{Red} and \second{Grey} mark the best and second best variants. \newcol{Teal column: ogbg-ppa, official OGB split, accuracy in percentage points.} $^{*}p<0.05$, $^{**}p<0.01$, $^{***}p<0.001$ (paired $t$-test vs.\ unmodified GIN); unmarked values are not significant. Full statistics in Appendix~\ref{app:ablation}.}
\label{tab:descriptor}
\small
\begin{tabular}{llcccc}
\toprule
 & \textbf{Variant} & \textbf{MUTAG} & \textbf{PROTEINS} & \textbf{COLLAB} & \newcol{\textbf{ogbg-ppa}} \\
\midrule
 & GIN (unmodified) & $0.832 \pm 0.084$ & $0.740 \pm 0.034$ & $0.810 \pm 0.029$ & \newcol{$67.84 \pm 0.917$} \\

\textit{UTS-Aug} GIN\ldelim\{{3}{2mm}[]
& Embedding-UTS & $0.872 \pm 0.079$ & $0.751 \pm 0.031^{*}$ & $0.823 \pm 0.028^{**}$ & \newcol{$69.27 \pm 0.903^{*}$} \\
& Graph-UTS & \best{0.890 $\pm$ 0.076} & \second{0.759 $\pm$ 0.040$^{*}$} & \second{0.832 $\pm$ 0.029$^{**}$} & \newcol{\second{$70.20 \pm 0.951^{**}$}} \\
& Dual UTS & \second{0.878 $\pm$ 0.078} & \best{0.763 $\pm$ 0.036$^{*}$} & \best{0.837 $\pm$ 0.028$^{***}$} & \newcol{\best{$71.06 \pm 0.934^{**}$}} \\
\bottomrule
\end{tabular}
\end{table}

\rebuttal{The effectiveness of $\Phi_{\mathrm{emb}}$ alone is more modest than that of the graph signature but is positive across all datasets: it improves over the unmodified GIN on MUTAG ($+4.0\%$, not significant), PROTEINS ($+1.1\%$, significant), COLLAB ($+1.3\%$, significant), \newcol{and ogbg-ppa ($+1.43$ points, significant)}. Combining both signatures (Dual-UTS) achieves the best performance on PROTEINS, COLLAB, \newcol{and ogbg-ppa}, demonstrating that the static and dynamic signatures provide complementary structural information on these datasets, while Graph-UTS alone remains the strongest single descriptor on MUTAG. \newcol{Notably, the descriptor gains on ogbg-ppa hold under a species-level train/test split that requires generalizing to entirely unseen species, suggesting the topological signal captured by UTS is not merely fitting to species-specific idiosyncrasies present in the training distribution.}}

\subsection{Effect of UTS-Reg (Topology-Preserving Regularization)}
\label{subsec:loss_results}

\rebuttal{The previous experiment showed that Embedding-UTS is sensitive to representation quality. We therefore evaluate whether UTS-Reg improves it by preserving topological fidelity during optimization. Table~\ref{tab:loss} shows the two regularizers behave complementarily: the evolution loss gives the largest gain on MUTAG ($+1.9\%$, not significant), while the alignment loss performs best on PROTEINS ($+1.4\%$, significant) \newcol{and on ogbg-ppa ($+1.70$ points, significant)}. Neither loss improves over baseline on COLLAB, and both individual losses are significantly negative there, indicating the benefit of topology-preserving supervision is dataset dependent -- consistent with the layer-wise Oversmoothing Index analysis in Appendix~\ref{app:osi_results}, which shows the same evolution loss reduces representation collapse on MUTAG but increases it on PROTEINS and COLLAB. \newcol{On ogbg-ppa, all three loss variants improve over baseline and are significant, unlike on COLLAB; we discuss this cross-dataset inconsistency further in \S\ref{subsec:discussion}.} We discuss the theoretical implications of this dataset-dependence in Section~\ref{sec:theory}.}

\begin{table}[h]
\centering
\caption{ Comparison of \textit{UTS-Reg} GIN backbone trained with the topological evolution loss ($\mathcal{L}_{\mathrm{topo-evol}}$), the topological alignment loss ($\mathcal{L}_{\mathrm{topo-align}}$), and their combination, with the unregularized GIN. \best{Red} and \second{Grey} mark the best and second best variant. \newcol{Teal column: ogbg-ppa.} Significance markers as in Table~\ref{tab:descriptor}, paired vs.\ unregularized GIN. Full statistics in Appendix~\ref{app:ablation}.}
\label{tab:loss}
\small
\setlength{\tabcolsep}{1.5pt}
\begin{tabular}{llcccc}
\toprule
&\textbf{Variant} & \textbf{MUTAG} & \textbf{PROTEINS} & \textbf{COLLAB} & \newcol{\textbf{ogbg-ppa}} \\
\midrule
&GIN (unregularized)                 & $0.832 \pm 0.084$ & \second{$0.740 \pm 0.034$} & \best{0.810 $\pm$ 0.029} & \newcol{$67.84 \pm 0.917$} \\
\textit{UTS-Reg} GIN\ldelim\{{3}{2mm}[]
&$\mathcal{L}_{\mathrm{topo-evol}}$  & \best{0.850 $\pm$ 0.085} & $0.738 \pm 0.034^{*}$ & $0.805 \pm 0.030^{*}$ & \newcol{\second{$68.36 \pm 0.982^{*}$}} \\
&$\mathcal{L}_{\mathrm{topo-align}}$  & \second{0.842 $\pm$ 0.088} & \best{0.755 $\pm$ 0.033$^{*}$} & $0.806 \pm 0.030^{*}$ & \newcol{\best{$69.54 \pm 0.941^{*}$}} \\
&Combination        & $0.841 \pm 0.084$ & 0.736 $\pm$ 0.035$^{*}$ & \second{0.809 $\pm$ 0.032$^{*}$} & \newcol{$68.15 \pm 1.006^{*}$} \\
\bottomrule
\end{tabular}
\end{table}

\rebuttal{Combining both regularizers, moreover, does not consistently outperform either alone, indicating that the two losses contribute differently across graph domains and their interaction requires dataset-specific balancing. This mirrors Section~\ref{subsec:descriptor_results}, where the embedding signature proved more sensitive to representation quality than the graph signature. \newcol{On ogbg-ppa, the combination is also weaker than the alignment loss alone, though it remains significantly positive, unlike the negative combination effect observed on PROTEINS and COLLAB.}}

\subsection{Evaluation of UTS-Pool (Topology-aware pooling)}
\label{subsec:pooling_results}

\rebuttal{Table~\ref{tab:pooling} compares UTS-Pool against TOGL, SAGPool, and TopKPool under the same stratified 10-fold CV protocol (MUTAG/PROTEINS/COLLAB) and the official OGB split (ogbg-ppa). UTS-Pool achieves the best mean accuracy on PROTEINS and COLLAB, with the advantage over TopKPool and SAGPool statistically significant on both datasets, and a smaller but still significant advantage over TOGL ($+0.1$ percentage points on both). On MUTAG, TOGL is nominally ahead of UTS-Pool ($0.840$ vs.\ $0.835$), though the difference does not reach significance; UTS-Pool remains ahead of SAGPool and TopKPool on MUTAG, but this gap is likewise not significant. \newcol{On ogbg-ppa UTS-Pool remains significantly ahead of SAGPool and TopKPool on ogbg-ppa. TOGL is slightly but signficantly better than UTS-Pool ($68.97$ vs.\ $68.85$, $\Delta=-0.12$, $p=0.0008$). }Notably, the COLLAB gains are obtained using the lightweight geometric scorer $\phi^{\mathrm{light}}_v$, demonstrating that even a computationally efficient local geometric proxy provides sufficient structural information to guide scalable and robust graph coarsening \newcol{-- the same lightweight scorer is used for UTS-Pool on ogbg-ppa, given its even larger average graph size (243.4 nodes)}.}

\begin{table}[h]
\centering
\caption{ Accuracy comparison of UTS-Pool against TOGL, TopKPool, and SAGPool, same GIN backbone. \best{Red} and \second{Grey} mark the best and second best variant. \newcol{Teal column: ogbg-ppa.} $^{*}p<0.05$, $^{**}p<0.01$, $^{***}p<0.001$ (paired $t$-test, variant vs.\ UTS-Pool); unmarked values are not significant. Full statistics in Appendix~\ref{app:ablation}. $^{\dagger}$Our faithful reimplementation of TOGL~\citep{horn2022topologicalgraphneuralnetworks}, adapted to this codebase; see Appendix~\ref{app:pooling} for implementation details.}
\label{tab:pooling}
\small
\begin{tabular}{lcccc}
\toprule
\textbf{Variant} & \textbf{MUTAG} & \textbf{PROTEINS} & \textbf{COLLAB} & \newcol{\textbf{ogbg-ppa}} \\
\midrule
TopKPool       & $0.829 \pm 0.079$ & $0.741 \pm 0.033^{***}$ & $0.818 \pm 0.029^{*}$ & \newcol{$67.53 \pm 0.988^{***}$} \\
SAGPool        & $0.831 \pm 0.080$ & $0.742 \pm 0.034^{**}$  & $0.819 \pm 0.029^{*}$ & \newcol{$67.88 \pm 0.963^{***}$} \\
TOGL$^{\dagger}$ & \best{0.840 $\pm$ 0.083} & \second{0.744 $\pm$ 0.033$^{*}$} & \second{0.820 $\pm$ 0.029$^{*}$} & \newcol{\best{$68.97 \pm 0.926^{***}$}} \\
UTS-Pool       & \second{0.835 $\pm$ 0.079} & \best{0.745 $\pm$ 0.033} & \best{0.821 $\pm$ 0.029} & \newcol{\second{$68.85 \pm 0.891$}} \\
\bottomrule
\end{tabular}

\end{table}

Finally, Figure~\ref{fig:uts_evolution_imp} illustrates the evolution of representative components of the embedding signature across GIN layers on the MUTAG dataset. As message passing progresses, the $H_0$ mean lifetime decreases from $9.08$ to $3.51$, while the spectral gap $\lambda_1$ shrinks from $1.26$ to $0.81$, indicating progressive contraction of the embedding topology associated with oversmoothing. In contrast, the $H_1$ mean lifetime remains nearly unchanged, suggesting that higher-order topological structure is largely preserved. These trends demonstrate that UTS provides an interpretable characterization of representation dynamics during GNN optimization. Further details on the proposed \textit{Oversmoothing Index} (OSI) formulation, implementation, and experimental analysis are provided in Appendix~\ref{app:osi_results}.

\begin{figure}[t]
    \centering
    \includegraphics[width=\textwidth]{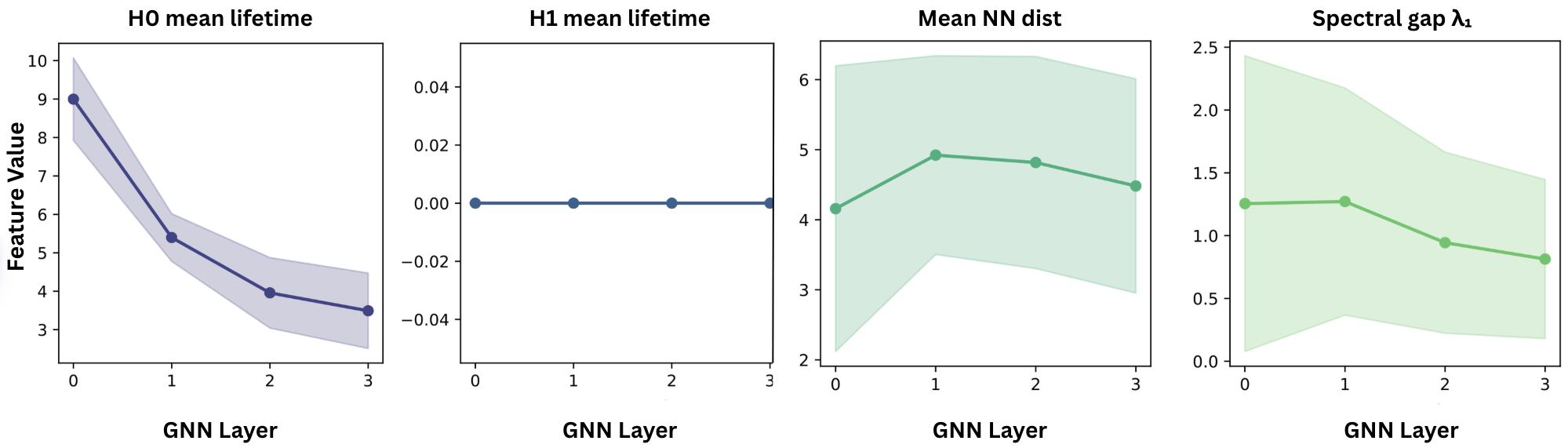}
    \caption{Evolution of some components of embedding signature ($\Phi_{\mathrm{emb}}$) across GIN layers on the MUTAG dataset. The sharp decay in $H_{0}$ mean lifetime and the shrinking spectral gap ($\lambda_{1}$) empirically diagnose representation collapse, where distinct node embeddings oversmooth into an indistinguishable manifold}
    \label{fig:uts_evolution_imp}
\end{figure}

\section{Theoretical Analysis}
\label{sec:theory}

We now establish the theoretical properties of the proposed framework. Complete proofs are deferred to Appendix~\ref{app:proofs}.
\subsection{Expressive Power of Dual Unified Topological Signatures}

\begin{theorem}
\label{thm:expressivity}
Let $\mathcal{A}$ denote a message-passing GNN whose expressive power is bounded by the $1$-Weisfeiler--Lehman (1-WL) test. Augmenting $\mathcal{A}$ with the graph signature $\Phi_{\mathrm{grph}}$ yields a graph representation capable of distinguishing graph pairs that are indistinguishable under the 1-WL test whenever their global topological signatures differ.
\end{theorem}


\subsection{Topology-Preserving Regularization}
\rebuttal{\begin{proposition}[Topo-Evolution Loss: Heuristic Motivation]
\label{prop:topo-evol-heuristic}
The topo-evolution loss $\mathcal{L}_{\mathrm{smooth}}$
(Eq.~\ref{eq:layer_smooth}) heuristically discourages representation
collapse by penalizing layer-to-layer divergence of the embedding-space
topological signature $\mathbf{s}^{(l)}$. This is a first-order,
directional argument, not a proof that $\mathcal{L}_{\mathrm{smooth}}$
enforces a variance bound or guarantees reduced Oversmoothing Index
(Eq.~\ref{eq:osi}); the effect is empirically dataset-dependent (see
Appendix~\ref{app:prop1_full} for the full justification and OSI
comparison across datasets).
\end{proposition}}





\subsection{Computational Complexity}
The graph signature $\Phi_{\mathrm{grph}}$ is computed once per graph and cached, while the embedding signature $\Phi_{\mathrm{emb}}^{\mathrm{diff}}$ is recomputed at every layer, both at cost $O(n^3)$ dominated by persistence and spectral computations. The lightweight pooling proxy $\phi_v^{\mathrm{light}}$ reduces this to $O(d\,n\,\bar\delta^2)$ per pooling layer, linear in $n$ at fixed local density. A full derivation is given in Appendix~\ref{app:complexity}.

\section{Conclusion and Future Work}
\label{sec:conclusion}

We introduced a unified topology-aware framework for graph representation learning based on Dual Unified Topological Signatures, jointly modeling the intrinsic topology of the input graph and the evolving topology of the learned embedding space throughout the GNN pipeline. Topology improves graph learning through three complementary mechanisms: topology-augmented representations, topology-preserving regularization, and topology-aware hierarchical pooling, with consistent improvements over the GIN backbone across molecular, biological, and social-network benchmarks.

Future work includes differentiable relaxations of the UTS-Pool selection operator and extending the framework to heterogeneous and temporal graph learning. \rebuttal{The full Embedding-UTS signature incurs O(n³)-per-layer cost, which limits its direct application to long-range or web-scale graph benchmarks; developing a more computationally efficient variant of the signature is a natural direction for extending this work to such settings.}
\printbibliography

\appendix
\section{Appendix}
\subsection{Full Derivations of Universal Topological Signatures}
\label{app:uts}

This appendix gives the complete definitions, differentiability
remarks, and computational notes underlying the embedding- and
graph-level signature maps introduced in
Section~\ref{subsec:embed_uts} and Section~\ref{subsec:graph_uts}.

\subsection{Embedding UTS: \texorpdfstring{$\Phi_{\mathrm{emb}}$}{PhiEmb}}
\label{app:embed_uts_full}

\subsubsection{Pairwise Distance Matrix}
\begin{equation}
    D_{uv} = \left\|\mathbf{h}_u - \mathbf{h}_v\right\|_2
           = \sqrt{\sum_{i=1}^{d_l}(h_{ui} - h_{vi})^2 + \varepsilon},
    \quad \varepsilon = 10^{-8},
    \label{eq:dist_mat}
\end{equation}
where $\varepsilon$ ensures numerical stability when
$\mathbf{h}_u = \mathbf{h}_v$ (in case of oversmoothing).



\subsubsection{Local Geometry Features (Dimensions 7--10)}
Let $\mathcal{N}_k(v)$ denote the $k$ nearest neighbours of node $v$
excluding itself. Define
\begin{align}
    \mu_{\mathrm{nn}} &= \frac{1}{n}\sum_{v \in \mathcal{V}}
        \frac{1}{k}\sum_{u \in \mathcal{N}_k(v)} D_{vu},
        \label{eq:mean_nn} \\
    \sigma_{\mathrm{nn}} &= \sqrt{\frac{1}{nk}\sum_{v}\sum_{u \in
        \mathcal{N}_k(v)}\left(D_{vu} - \mu_{\mathrm{nn}}\right)^2},
        \label{eq:std_nn} \\
    \Delta &= \max_{u,v \in \mathcal{V}} D_{uv}, \label{eq:spread} \\
    \hat{d} &= \mathrm{median}_{v \in \mathcal{V}}
        \frac{\log r_2(v) + \varepsilon}{\log r_1(v) + \varepsilon},
        \label{eq:intrinsic_dim}
\end{align}
where $r_j(v)$ is the distance to the $j$-th nearest neighbour of $v$.
Equation~\eqref{eq:intrinsic_dim} provides a correlation-dimension
estimate of the intrinsic dimensionality of the embedding manifold, and
is used as written for $\Phi_{\mathrm{emb}}$ (exact).

\paragraph{Remark on Differentiability.}
$\mu_{\mathrm{nn}}$, $\sigma_{\mathrm{nn}}$, and $\Delta$
\eqref{eq:mean_nn}--\eqref{eq:spread} are differentiable with respect to
$\mathbf{H}^{(l)}$ through \eqref{eq:dist_mat} in both variants.
\rebuttal{The $\mathrm{median}$ in \eqref{eq:intrinsic_dim}, however, has
a degenerate gradient (non-zero only at the median element), so
$\Phi_{\mathrm{emb}}^{\mathrm{diff}}$ replaces it with a $\mathrm{mean}$
over the same log-ratio, distributing gradient across all $n$ nodes; see
Appendix~\ref{app:diff_uts} for the full justification.}

\subsubsection{Persistence Features (Dimensions 1-6)}

\rebuttal{For $\Phi_{\mathrm{emb}}$ (exact), dimensions 1-6 are computed
via standard Vietoris--Rips persistent homology on the detached point
cloud using GUDHI \citep{gudhi:urm} - identical in construction to the
persistence features of $\Phi_{\mathrm{grph}}$ (\S\ref{app:graph_uts_full}),
yielding exact mean lifetimes, persistence entropies, and integer Betti
numbers $\beta_0, \beta_1$. For $\Phi_{\mathrm{emb}}^{\mathrm{diff}}$, these
exact operations are replaced by the differentiable relaxations below,
since MST extraction and simplicial homology computation are not
differentiable.}

Let $\{w_e\}$ denote the upper-triangular entries of $\mathbf{D}$,
sorted in ascending order as $w_{(1)} \leq w_{(2)} \leq \cdots \leq
w_{\binom{n}{2}}$. We define the \emph{H0 lifetime surrogate} as the
$n-1$ smallest edge weights:
\begin{equation}
    \ell_i^{(0)} = w_{(i)}, \quad i = 1, \ldots, n-1.
    \label{eq:h0_lifetimes}
\end{equation}
This approximates the persistence diagram of the 0-dimensional
homology of the Vietoris–Rips complex: in an exact Rips filtration,
connected components merge at minimum spanning tree edge weights.
Equation~\eqref{eq:h0_lifetimes} is a differentiable relaxation
since MST extraction (Kruskal's algorithm) is non-differentiable.

For 1-dimensional homology we enumerate all triangles
$\binom{\mathcal{V}}{3}$ and define, for each triangle
$\tau = (u, v, w)$ with sorted edge lengths
$e_1(\tau) \leq e_2(\tau) \leq e_3(\tau)$,
\begin{equation}
    \ell^{(1)}_\tau = \max\!\left(0,\, e_3(\tau) - e_2(\tau)\right).
    \label{eq:h1_lifetimes}
\end{equation}
This follows from the Vietoris–Rips persistence pairing: a 1-cycle
is born when the second-longest triangle edge enters the filtration
and dies when the longest edge fills the triangle.

\rebuttal{\begin{remark}[When are surrogates used?]
Surrogates ($\Phi_{\mathrm{emb}}^{\mathrm{diff}}$) are used only where
gradients must propagate through the signature: the topology-preserving
regularizers (\S\ref{subsec:regularization}) and the UTS-Pool node scorer
(\S\ref{subsec:pooling}). UTS-Aug (\S\ref{subsec:descriptor}) uses the
exact signature $\Phi_{\mathrm{emb}}$, computed on detached embeddings
with no gradient requirement. See Appendix~\ref{app:diff_uts} for the
complete differentiable-surrogate derivation.
\end{remark}}
\begin{remark}[Computational cost of H1]
Triangle enumeration scales as $O(n^3)$. We cap computation at
$n \leq n_{\max}$ (default $n_{\max} = 50$); for larger point clouds
the H1 features are set to zero. This is appropriate for large social
network graphs (e.g.\ COLLAB) where oversmoothing renders the
embedding point cloud near-degenerate in any case. \rebuttal{This cap
applies to $\Phi_{\mathrm{emb}}^{\mathrm{diff}}$ wherever it is computed
per layer -- both the topology-evolution loss (\S\ref{subsec:regularization})
and the UTS-Pool node scorer (\S\ref{subsec:pooling}) -- and does not
apply to UTS-Aug, which uses the exact signature $\Phi_{\mathrm{emb}}$
computed once on detached embeddings.}
\end{remark}

From \eqref{eq:h0_lifetimes} and \eqref{eq:h1_lifetimes} we
compute mean lifetimes and persistence entropies:
\begin{align}
    \bar{\ell}^{(k)} &= \frac{1}{|\mathcal{L}^{(k)}|}
        \sum_i \ell_i^{(k)}, \\
    H^{(k)} &= -\sum_i \hat{p}_i^{(k)} \log \hat{p}_i^{(k)}, \quad
        \hat{p}_i^{(k)} = \frac{\ell_i^{(k)}}
        {\sum_j \ell_j^{(k)} + \varepsilon}.
\end{align}
Soft Betti numbers are defined as
\begin{equation}
    \tilde{\beta}_k = \sum_i \sigma\!\left(
        10\left(\ell_i^{(k)} - \mu_{\mathrm{nn}}\right)\right),
    \label{eq:soft_betti}
\end{equation}
where $\sigma$ is the sigmoid function and $\mu_{\mathrm{nn}}$ serves
as an adaptive threshold.

\subsubsection{Spectral Features (Dimensions 11--14)}

\rebuttal{For $\Phi_{\mathrm{emb}}$ (exact), the adjacency is a standard
binary $k$-NN graph and $\tilde{\mathbf{L}}$ is the unnormalised
combinatorial Laplacian, with eigenvalues computed via
\texttt{scipy.linalg.eigh}; no soft kernel or regularisation is applied,
since differentiability is not required. For $\Phi_{\mathrm{emb}}^{\mathrm{diff}}$,
we construct a soft $k$-NN adjacency matrix via a Gaussian kernel instead,
to keep the spectral features differentiable:}

Construct a soft $k$-NN adjacency matrix via a Gaussian kernel:
\begin{equation}
    A_{uv}^{\mathrm{soft}} = \exp\!\left(
        -\frac{D_{uv}^2}{2(\mu_{\mathrm{nn}} + \varepsilon)^2}
    \right)\cdot\mathds{1}[u \neq v].
    \label{eq:soft_adj}
\end{equation}
The bandwidth $\mu_{\mathrm{nn}}$ is kept in the computation graph
(not detached), ensuring gradients flow through the spectral features
back to $\mathbf{H}^{(l)}$. Let
$\mathbf{L}_{\mathrm{soft}} = \mathbf{D}_{\mathrm{soft}} -
\mathbf{A}^{\mathrm{soft}}$
be the unnormalised Laplacian, regularised as
\begin{equation}
    \tilde{\mathbf{L}} = \frac{\mathbf{L}_{\mathrm{soft}} +
        \delta\mathbf{I}}{\max_v d_v^{\mathrm{soft}} + \varepsilon},
    \quad \delta = 10^{-4},
    \label{eq:reg_laplacian}
\end{equation}
where $d_v^{\mathrm{soft}} = \sum_u A_{uv}^{\mathrm{soft}}$.
Division by $\max_v d_v^{\mathrm{soft}}$ normalises the matrix to
$[0,1]$ scale, preventing ill-conditioning on dense graphs
(e.g.\ COLLAB with $\bar{d} \approx 8.9$).

Let $0 \leq \lambda_1 \leq \lambda_2 \leq \cdots \leq \lambda_n$ be
the eigenvalues of $\tilde{\mathbf{L}}$, computed via
\texttt{torch.linalg.eigh} (fully differentiable for symmetric
matrices). The spectral gap $\lambda_1$, and $\lambda_2, \lambda_3$
approximate the Fiedler value and higher connectivity information.
Spectral entropy is
\begin{equation}
    H_{\mathrm{spec}} = -\sum_i \hat{q}_i \log \hat{q}_i, \quad
    \hat{q}_i = \frac{\max(\lambda_i, \varepsilon)}
        {\sum_j \max(\lambda_j, \varepsilon)}.
\end{equation}

\paragraph{Stability.}
The spectral entropy $H_{\mathrm{spec}}$ is Lipschitz continuous
in the entries of $\tilde{\mathbf{L}}$ by the Weyl perturbation
theorem and the Lipschitz continuity of the entropy function on the
probability simplex.

\subsection{Graph UTS: \texorpdfstring{$\Phi_{\mathrm{grph}}$}{PhiGrph}}
\label{app:graph_uts_full}

\subsubsection{Ollivier-Ricci Curvature (Dimensions 1-4)}

For an edge $(u,v) \in \mathcal{E}$, the Ollivier-Ricci curvature
\citep{ollivier2007riccicurvaturemarkovchains} is
\begin{equation}
    \kappa(u,v) = 1 - \frac{W_1(\mu_u, \mu_v)}{d_G(u,v)},
    \label{eq:orc}
\end{equation}
where $d_G(u,v)$ is the geodesic distance, $\mu_v$ is the
$\alpha$-lazy random walk measure at $v$:
\begin{equation}
    \mu_v(w) = \begin{cases}
        \alpha & w = v \\
        (1 - \alpha)/\deg(v) & (v,w) \in \mathcal{E} \\
        0 & \text{otherwise},
    \end{cases}
\end{equation}
and $W_1(\mu_u, \mu_v)$ is the Wasserstein-1 distance computed via
the Sinkhorn approximation \citep{cuturi2013sinkhorndistanceslightspeedcomputation}:
\begin{equation}
    W_1^\varepsilon(\mu_u, \mu_v) \approx
    \langle \mathbf{u} \otimes \mathbf{K} \odot \mathbf{C} \otimes
    \mathbf{v} \rangle,
\end{equation}
where $\mathbf{K} = \exp(-\mathbf{C}/\varepsilon)$ is the Gibbs
kernel, $\mathbf{C}$ is the cost matrix of geodesic distances
restricted to the supports of $\mu_u$ and $\mu_v$, and $\mathbf{u},
\mathbf{v}$ are Sinkhorn scaling vectors. This replaces the
C\texttt{++} implementation of \citep{ni2019communitydetectionnetworksricci}, eliminating
segmentation faults on varied graph sizes while remaining
differentiable through the Sinkhorn iterations if required.

From the edge curvatures $\{\kappa_e\}_{e \in \mathcal{E}}$ we
extract
\begin{equation}
    \phi_{\mathrm{ORC}} = \left[
        \bar\kappa,\;
        \mathrm{Var}(\kappa),\;
        \min_e \kappa_e,\;
        \frac{|\{e : \kappa_e < 0\}|}{|\mathcal{E}|}
    \right] \in \mathbb{R}^4.
\end{equation}
Edges with $\kappa_e > 0$ lie within dense clusters; edges with
$\kappa_e < 0$ are inter-cluster bridges. The fraction of negatively
curved edges thus serves as a proxy for community structure
\citep{ni2019communitydetectionnetworksricci}.

\subsubsection{Forman-Ricci Curvature (Dimensions 5-8)}

The combinatorial Forman-Ricci curvature \citep{article}
of an edge $e = (u,v)$ is
\begin{equation}
    F(e) = w_e\!\left(
        \frac{w_u}{w_e} + \frac{w_v}{w_e}
        - \sum_{e' \sim u, e' \neq e}\sqrt{\frac{w_e}{w_{e'}}}
        - \sum_{e' \sim v, e' \neq e}\sqrt{\frac{w_e}{w_{e'}}}
    \right),
\end{equation}
where $w_e, w_v$ are edge and vertex weights (unity for unweighted
graphs). We extract $\phi_{\mathrm{FRC}} = [\bar{F}, \mathrm{Var}(F),
\min_e F_e, \max_e F_e] \in \mathbb{R}^4$.

\subsubsection{Distance Features (Dimensions 9-10)}

Let $\ell_{G}(u,v)$ denote the shortest-path distance in the largest
connected component of $G$. We extract the mean path length
$\bar{\ell}_G = \frac{1}{n(n-1)}\sum_{u \neq v} \ell_G(u,v)$ and
diameter $\mathrm{diam}(G) = \max_{u,v} \ell_G(u,v)$.

\subsubsection{Spectral Features (Dimensions 11-13)}

Let $\mathbf{L}_G = \mathbf{D}_G - \mathbf{A}$ be the combinatorial
graph Laplacian and $0 = \lambda_1^G \leq \lambda_2^G \leq \cdots$
its eigenvalues. We extract the Fiedler value $\lambda_2^G$
(algebraic connectivity), the largest eigenvalue $\lambda_n^G$, and
spectral entropy \rebuttal{$H_{\mathrm{spec}}^G$, computed by the same
entropy formula as $H_{\mathrm{spec}}$ (Eq.~\ref{eq:embed_uts_vec}'s
spectral group) but applied directly to the exact eigenvalues of
$\mathbf{L}_G$ above -- no soft kernel or regularisation is used, since
$\Phi_{\mathrm{grph}}$ requires no gradient (\S\ref{app:graph_uts_full},
Persistence Features).}

\subsubsection{Persistence Features (Dimensions 14-17)}

Applying the Vietoris-Rips filtration to the shortest-path distance
matrix $(\ell_G(u,v))_{u,v}$ via GUDHI \citep{gudhi:urm}, we extract
mean lifetimes and persistence entropies of $H_0$ and $H_1$
homology groups, computed exactly (not surrogates) since
$\Phi_{\mathrm{grph}}$ is non-differentiable by design.

\subsubsection{Structural Features (Dimensions 18-27)}

Degree statistics $[\bar{d}, \mathrm{Var}(d), \max_v d_v]$,
clustering coefficients $[\bar{c}, \mathrm{Var}(c)]$,
betweenness centrality $[\bar{b}, \mathrm{Var}(b)]$,
mean closeness centrality $\bar{q}$, number of connected
components $N_{\mathrm{cc}}$, and largest component ratio
$|\mathcal{C}_{\max}|/n$.

\subsection{Differentiable Surrogate Formulation of \texorpdfstring{$\Phi_{\mathrm{emb}}$}{PhiEmb}}
\label{app:diff_uts}

The readout descriptor (\S4.1) computes $\Phi_{\mathrm{emb}}$ on
$\mathbf{H}^{(L)}.\texttt{detach()}$ using GUDHI for exact persistent
homology.  However, the regularisation losses (\S4.2) and pooling scorer
(\S4.3) require gradients to flow back through the topological signature
into the GIN encoder. We therefore maintain a fully differentiable re-implementation, $\Phi_{\mathrm{emb}}^{\mathrm{diff}}$, in PyTorch
(\texttt{diff\_uts.py}) that produces the same 14-dimensional vector as $\Phi_{\mathrm{emb}}$ without any NumPy or GUDHI operations. The two implementations agree closely on well-separated point
clouds and diverge gracefully in the oversmoothing regime where exact
persistence is ill-defined regardless.

Below we justify each non-obvious design choice.

\paragraph{Pairwise distances.}
We compute
\begin{equation}
  D_{uv} = \sqrt{\textstyle\sum_i (h_{ui} - h_{vi})^2 + \varepsilon},
  \quad \varepsilon = 10^{-8},
\end{equation}
rather than \texttt{torch.cdist}.  The gradient of \texttt{cdist} at
$D_{uv} = 0$ is undefined when $\mathbf{h}_u = \mathbf{h}_v$, which occurs
frequently in early training before the encoder has differentiated nodes.
The $\varepsilon$ inside the square root regularises the gradient to
$\partial D_{uv} / \partial \mathbf{h}_u = (\mathbf{h}_u -
\mathbf{h}_v) / D_{uv}$, which remains bounded.

\paragraph{$H_0$ surrogate.}
Exact $H_0$ persistence requires a minimum spanning tree.  MST extraction
(Kruskal's algorithm) involves \texttt{argsort} whose gradient is non-unique
at ties and propagates to only one entry per comparison.  We use the
$N\!-\!1$ smallest upper-triangle pairwise distances $w_{(1)} \le \ldots \le
w_{(N-1)}$ as a differentiable proxy for component lifetimes
(cf.~Eq.~\ref{eq:h0_lifetimes}).  This approximation is tight when the
point cloud is well-separated: in an exact Rips filtration, components
merge precisely at MST edge weights, which are the $N\!-\!1$ smallest
distances when no two inter-component distances are equal.
The gradient of \texttt{torch.sort} is well-defined (it equals the
identity on the sorted entries) and distributes across the source entries.

\paragraph{$H_1$ surrogate.}
We enumerate all $\binom{N}{3}$ triangles and compute per-triangle lifetimes
$\ell^{(1)}_\tau = \max(0,\, e_3(\tau) - e_2(\tau))$ where $e_2, e_3$ are
the second- and third-longest edge lengths of triangle $\tau$
(cf.~Eq.~\ref{eq:h1_lifetimes}).  This quantity is an upper bound on the
true $H_1$ lifetime in the Vietoris--Rips filtration: under Rips, a 1-cycle
is born at most when the second-longest triangle edge enters and dies when
the longest edge fills the triangle.  Triangle enumeration is $O(N^3)$ in
memory; we cap it at $N \le 50$, setting $\bar{\ell}^{(1)} = H^{(1)} =
\tilde\beta_1 = 0$
for larger point clouds.  For the datasets in our experiments, only
COLLAB graphs exceed this threshold after the 2-hop neighbourhood cap of
30 nodes is applied in \S4.3.

\paragraph{Soft Betti numbers.}
The exact Betti numbers $\beta_k$ are integers and non-differentiable.  We
replace them with sigmoid-smoothed counts
$\tilde\beta_k = \sum_j \sigma(10(\ell_j^{(k)} - \mu_{\mathrm{nn}}))$,
where the threshold $\mu_{\mathrm{nn}}$ adapts to the point cloud density
and the scale factor 10 provides sharpness comparable to a step function
while keeping gradients non-zero.

\paragraph{Spectral features.}
The binary $k$-NN adjacency used in the non-differentiable
$\Phi_{\mathrm{emb}}$ is replaced by a Gaussian kernel adjacency
$A_{uv}^{\mathrm{soft}} = \exp(-D_{uv}^2 / 2(\mu_{\mathrm{nn}} +
\varepsilon)^2) \cdot \mathds{1}[u \ne v]$ (cf.~Eq.~\ref{eq:soft_adj}).
The bandwidth $\mu_{\mathrm{nn}}$ remains in the computation graph so
gradients flow through the spectral features.  We add $10^{-6}\mathbf{I}$
to the adjacency before forming the Laplacian to prevent degenerate (zero)
diagonal entries that cause NaN in \texttt{torch.linalg.eigh}.  We use
\texttt{eigh} rather than \texttt{eig}: for symmetric matrices it
guarantees real eigenvalues and implements the analytic gradient
$\partial\lambda_i / \partial \mathbf{L} = \mathbf{q}_i \mathbf{q}_i^\top$
via the eigenvalue equation, which is numerically stable for distinct
eigenvalues.

\paragraph{Intrinsic dimensionality.}
We use \texttt{mean} rather than \texttt{median} over the log-ratio
$\log r_i^{(2)} / \log r_i^{(1)}$ (cf.~Eq.~\ref{eq:intrinsic_dim}).
The gradient of \texttt{median} is non-zero only at the median element,
making it effectively a single-node pass-through and providing no useful
learning signal.  The \texttt{mean} distributes gradient uniformly across
all $N$ nodes.

\subsection{Architecture and Training Details}
\label{app:impl}

\paragraph{Full readout dimensionalities.}
The classifier MLP input dimension varies by ablation variant.
With $d = 128$ (hidden dim), the base readout
$\mathbf{z}_{\mathrm{struct}} \in \mathbb{R}^{3d = 384}$.
Appending $\Phi_{\mathrm{emb}}$ adds 14 dimensions;
appending $\Phi_{\mathrm{grph}}$ adds 27 dimensions.
Table~\ref{tab:dims} summarises the classifier input dimension and
the number of parameters in the classifier head per variant
(backbone parameters are identical across all variants).

\begin{table}[h]
\centering
\small
\caption{Classifier input dimension and head parameter count per variant
($d=128$, $C$ = num classes).  Backbone parameters are identical across
all variants and not included.}
\label{tab:dims}
\renewcommand{\arraystretch}{1.1}
\begin{tabular}{lcc}
\toprule
\textbf{Variants} & \textbf{Readout dim} & \textbf{Head params (approx.)} \\
\midrule
GIN and UTS-Pool & 384 & $384{\times}128 + 128{\times}64 + 64{\times}C$ \\
Embedding-UTS                              & 398 & $398{\times}128 + \ldots$ \\
Graph-UTS                               & 411 & $411{\times}128 + \ldots$ \\
Dual UTS                     & 425 & $425{\times}128 + \ldots$ \\
\bottomrule
\end{tabular}
\end{table}

\noindent
\rebuttal{Crucially, \textit{UTS-Reg} $L_{topo-evol}$, $L_{topo-align}$ and combination of both objectives all share the same architecture as GIN (unregularized) baseline.
Any accuracy differences between these variants and baseline are therefore
attributable entirely to the training objective, not to additional model
capacity. }

\paragraph{TopoRegLoss projection.}
$\Phi_{\mathrm{grph}}(G) \in \mathbb{R}^{27}$ and
$\Phi_{\mathrm{emb}}^{\mathrm{diff}}(\mathbf{H}^{(L)}) \in \mathbb{R}^{14}$
live in spaces of different dimension and scale.  The learnable projection
$\mathbf{W}_{\mathrm{align}} \in \mathbb{R}^{14 \times 27}$ in
$\mathcal{L}_{\mathrm{reg}}$ (Eq.~\ref{eq:topo_reg}) is trained jointly
with the encoder, finding the optimal linear alignment between structural
and embedding-space topology.  This avoids hand-crafted feature matching
while remaining interpretable: the learned projection weights reveal which
structural features are most predictive of embedding-space topology.

\paragraph{UTS-Pool gradient flow.}
The node scorer in \S4.3 computes importance scores $\alpha_v$ via
$\Phi_{\mathrm{emb}}^{\mathrm{diff}}$ and an MLP $f_\theta$, but the
subsequent top-$k$ selection returns integer indices.  Consequently, no
gradient flows back through the selection decision itself; the scorer
receives gradient only indirectly from the downstream cross-entropy loss
through the embeddings of selected nodes.  The scorer thus acts as an
adaptive heuristic rather than a fully learned selection criterion.
A soft-selection relaxation (e.g.\ Gumbel-softmax top-$k$) would enable
end-to-end training of the pooling decision and is left to future work.

\rebuttal{\paragraph{Training protocol.}
All models use AdamW \citep{loshchilov2019decoupledweightdecayregularization} with learning rate
$10^{-3}$ and weight decay $10^{-4}$, cosine annealing to $\eta_{\min} =
10^{-5}$ over 200 epochs, and gradient norm clipping at 2.0. Early
stopping with patience 30 on validation accuracy selects the final
checkpoint. We use 10-fold stratified cross-validation: for each fold,
the held-out 10\% serves as the test set, and the remaining 90\% is
further split 90/10 into training and validation sets, yielding an
81/9/10 train/validation/test split per fold, stratified by class label,
with every graph evaluated as a test example exactly once per seed.}

\paragraph{GraphUTS precomputation.}
Computing $\Phi_{\mathrm{grph}}$ per graph at every training step
is prohibitive: it involves all-pairs shortest paths ($O(n^2)$),
GUDHI Rips persistence, and optionally Ricci curvature.  We precompute
and cache $\Phi_{\mathrm{grph}}(G_i)$ for all graphs before the
first training epoch and look up vectors by dataset index during
training.  The cache is shared across all seeds and ablation variants
in a single run, so the precomputation cost is paid exactly once.

\subsection{Dataset Statistics and Descriptions}
\label{app:datasets}

Table~\ref{tab:dataset-stats} summarizes the three TU benchmark datasets~\citep{morris2020tudatasetcollectionbenchmarkdatasets} used in Section~\ref{sec:experiments}, spanning a molecular, a biological, and a social-network domain.

\begin{table}[h]
\centering
\caption{Statistics of the graph classification datasets used in our experiments.}
\label{tab:dataset-stats}
\small
\begin{tabular}{lccccc}
\toprule
\textbf{Dataset} & \textbf{Graphs} & \textbf{Classes} & \textbf{Avg. Nodes} & \textbf{Avg. Edges} & \textbf{Node Features} \\
\midrule
MUTAG    & 188  & 2 & 17.93 & 19.79   & 7 (categorical, atom type) \\
PROTEINS & 1113 & 2 & 39.06 & 72.82   & 3 (categorical, SSE type) \\
COLLAB   & 5000 & 3 & 74.49 & 2457.78 & none (node degree used) \\
\rebuttal{ogbg-ppa} & \rebuttal{158{,}100} & \rebuttal{37} & \rebuttal{243.4} & \rebuttal{2266.1} & \rebuttal{none (learned embedding)} \\
\bottomrule
\end{tabular}
\end{table}

\textbf{MUTAG} consists of 188 chemical compound graphs representing nitroaromatic and heteroaromatic compounds, with nodes as atoms and edges as chemical bonds. Node features are a 7-dimensional one-hot encoding of atom type. The binary classification task is to predict mutagenic effect on the Gram-negative bacterium \textit{Salmonella typhimurium}. MUTAG is the smallest and sparsest of the three datasets, with graphs ranging from 10 to 28 nodes.

\textbf{PROTEINS} contains 1113 graphs representing protein tertiary structures, with nodes as secondary structure elements (helices, sheets, and turns) and edges connecting elements that are sequential neighbors along the amino acid chain or spatially proximate in the folded structure. Node features are a 3-dimensional one-hot encoding of secondary structure element type. The binary classification task distinguishes enzymes from non-enzymes. Graphs are substantially larger and more variable in size than MUTAG, ranging up to 620 nodes.

\textbf{COLLAB} is a scientific-collaboration dataset derived from ego-networks of researchers in three fields: High Energy Physics, Condensed Matter Physics, and Astro Physics, with the 3-way classification task being to predict the field from the collaboration ego-network. Nodes represent researchers and edges represent co-authorship; the dataset carries no categorical node features, so node degree is used as the input feature following standard practice~\citep{xu2019how}. COLLAB graphs are markedly denser than MUTAG or PROTEINS, with average degree an order of magnitude higher, which motivates both the reduced GIN depth (Section~\ref{sec:experiments}) and the lightweight pooling scorer $\phi^{\mathrm{light}}_v$ (Section~\ref{subsec:pooling}) used for this dataset.

\rebuttal{\textbf{ogbg-ppa}~\citep{hu2020ogb} contains 158{,}100 protein-protein association graphs extracted from species spanning 37 taxonomic groups, with nodes representing proteins and edges representing biologically meaningful associations (7-dimensional edge features encode association type and confidence, not used in our GIN backbone). Nodes carry no input features; following the official OGB baseline for this dataset, we use a single learnable embedding shared across all nodes. The task is 37-way taxonomic-group classification. Unlike MUTAG, PROTEINS, and COLLAB, ogbg-ppa uses an official, externally-fixed \emph{species split}: validation and test graphs are drawn from species entirely unseen during training, even though each held-out species still belongs to one of the 37 training taxonomic groups. This makes ogbg-ppa an out-of-distribution generalization test rather than an in-distribution random split, and is the reason we depart from the 10-fold CV protocol used for the other three datasets (\S\ref{sec:experiments}). ogbg-ppa's graphs are also substantially larger than any TU dataset used in this work - more than $3\times$ COLLAB's average size and $32\times$ its graph count}

\section{Full Ablation Results}
\label{app:ablation}

\rebuttal{This appendix reports full aggregate statistics for every comparison summarized in Tables~\ref{tab:descriptor}--\ref{tab:pooling}. All results use 10-fold stratified cross-validation (MUTAG: 9 seeds $\times$ 10 folds = 90 evaluations per variant; PROTEINS and COLLAB: 5 seeds $\times$ 10 folds = 50 evaluations per variant). Confidence intervals are bootstrap 95\% CIs (1000 resamples) unless noted as paired-difference CIs. $p$-values are paired $t$-tests; the comparator for each subsection is stated in its introduction. All values below are rounded to three decimal places; deltas are computed from full-precision means before rounding, so a delta may differ slightly from the difference of the rounded means shown.}

\rebuttal{We allocate seeds according to measurement variance. Under 10-fold CV, each test fold contains roughly 19 graphs for MUTAG versus ~111 for PROTEINS and ~500 for COLLAB - MUTAG's per-fold accuracy estimates are inherently noisier simply because they're averaged over far fewer graphs. Using more seeds (9 vs. 5) on the smaller, higher-variance dataset is a targeted allocation of statistical power to where it's most needed. This is also why several MUTAG comparisons remain non-significant even at 90 total evaluations, while the same comparisons reach significance on PROTEINS/COLLAB at only 50 — the underlying noise floor, differs by dataset}

\subsection{UTS Descriptor Variants}
\label{app:ablation_uts}

\rebuttal{This subsection isolates the effect of adding Unified Topological Signature (UTS) descriptors, comparing Embedding-UTS, Graph-UTS, and Dual-UTS against the unmodified, unregularized GIN backbone.}

\begin{table}[h]
\centering
\caption{UTS descriptor variants vs.\ GIN (unmodified/unregularized) --- MUTAG.}
\label{tab:app_mutag_uts}
\footnotesize
\begin{tabular}{lcccc}
\toprule
\textbf{Variant} & \textbf{Mean $\pm$ Std} & \textbf{$\Delta$} & \textbf{95\% CI} & \textbf{$p$} \\
\midrule
GIN (unmodified/unregularized) & $0.832 \pm 0.084$ & --- & --- & --- \\
Embedding-UTS & $0.872 \pm 0.079$ & $+0.040$ & [0.863, 0.880] & 0.109 \\
Graph-UTS     & $0.890 \pm 0.076$ & $+0.058$ & [0.874, 0.906] & 0.102 \\
Dual-UTS      & $0.878 \pm 0.078$ & $+0.047$ & [0.870, 0.887] & 0.097 \\
\bottomrule
\end{tabular}
\end{table}

\begin{table}[h]
\centering
\caption{UTS descriptor variants vs.\ GIN (unmodified/unregularized) --- PROTEINS.}
\label{tab:app_proteins_uts}
\footnotesize
\begin{tabular}{lcccc}
\toprule
\textbf{Variant} & \textbf{Mean $\pm$ Std} & \textbf{$\Delta$} & \textbf{95\% CI} & \textbf{$p$} \\
\midrule
GIN (unmodified/unregularized) & $0.740 \pm 0.034$ & --- & --- & --- \\
Embedding-UTS & $0.751 \pm 0.031$ & $+0.011$ & [0.743, 0.760] & 0.042 \\
Graph-UTS     & $0.759 \pm 0.040$ & $+0.019$ & [0.748, 0.770] & 0.039 \\
Dual-UTS      & $0.763 \pm 0.036$ & $+0.023$ & [0.752, 0.773] & 0.028 \\
\bottomrule
\end{tabular}
\end{table}

\begin{table}[h]
\centering
\caption{UTS descriptor variants vs.\ GIN (unmodified/unregularized) --- COLLAB.}
\label{tab:app_collab_uts}
\footnotesize
\begin{tabular}{lcccc}
\toprule
\textbf{Variant} & \textbf{Mean $\pm$ Std} & \textbf{$\Delta$} & \textbf{95\% CI} & \textbf{$p$} \\
\midrule
GIN (unmodified/unregularized) & $0.810 \pm 0.029$ & --- & --- & --- \\
Embedding-UTS & $0.823 \pm 0.028$ & $+0.013$ & [0.816, 0.830] & 0.009 \\
Graph-UTS     & $0.832 \pm 0.029$ & $+0.021$ & [0.824, 0.839] & 0.002 \\
Dual-UTS      & $0.837 \pm 0.028$ & $+0.026$ & [0.829, 0.844] & $<0.001$ \\
\bottomrule
\end{tabular}
\end{table}

\begin{table}[h]
\centering
\caption{\rebuttal{UTS descriptor variants vs.\ GIN (unmodified/unregularized) --- ogbg-ppa. Official OGB species split, 3 seeds.}}
\label{tab:app_ppa_uts}
\footnotesize
\begin{tabular}{lcccc}
\toprule
\textbf{Variant} & \textbf{Mean $\pm$ Std} & \textbf{$\Delta$} & \textbf{95\% CI} & \textbf{$p$} \\
\midrule
\rebuttal{GIN (unmodified/unregularized)} & \rebuttal{$67.84 \pm 0.917$} & \rebuttal{---} & \rebuttal{---} & \rebuttal{---} \\
\rebuttal{Embedding-UTS} & \rebuttal{$69.27 \pm 0.903$} & \rebuttal{$+1.43$} & \rebuttal{[68.76, 69.78]} & \rebuttal{0.0168} \\
\rebuttal{Graph-UTS}     & \rebuttal{$70.20 \pm 0.951$} & \rebuttal{$+2.36$} & \rebuttal{[69.66, 70.74]} & \rebuttal{0.0041} \\
\rebuttal{Dual-UTS}      & \rebuttal{$71.06 \pm 0.934$} & \rebuttal{$+3.22$} & \rebuttal{[70.53, 71.59]} & \rebuttal{0.0012} \\
\bottomrule
\end{tabular}
\end{table}
 
\rebuttal{\textit{Summary.} All three variants improve significantly over baseline, with Dual-UTS strongest -- the same ordering seen on PROTEINS and COLLAB, now confirmed at $32\times$ their graph count. Unlike MUTAG, the ogbg-ppa gains reach significance at this seed count, consistent with the much larger fixed test split reducing per-seed variance.}
\subsection{Topological Loss Variants}
\label{app:ablation_loss}

\rebuttal{This subsection isolates the effect of the topology-aware auxiliary losses ($\mathcal{L}_{\mathrm{topo-evol}}$, $\mathcal{L}_{\mathrm{topo-align}}$, and their combination), compared against the unmodified, unregularized GIN backbone.}

\begin{table}[h]
\centering
\caption{Topological loss variants vs.\ GIN (unmodified/unregularized) --- MUTAG.}
\label{tab:app_mutag_loss}
\footnotesize
\begin{tabular}{lcccc}
\toprule
\textbf{Variant} & \textbf{Mean $\pm$ Std} & \textbf{$\Delta$} & \textbf{95\% CI} & \textbf{$p$} \\
\midrule
GIN (unmodified/unregularized) & $0.832 \pm 0.084$ & --- & --- & --- \\
$\mathcal{L}_{\mathrm{topo-evol}}$  & $0.850 \pm 0.085$ & $+0.019$ & [0.834, 0.867] & 0.210 \\
$\mathcal{L}_{\mathrm{topo-align}}$ & $0.842 \pm 0.088$ & $+0.011$ & [0.824, 0.859] & 0.205 \\
Combination                          & $0.841 \pm 0.084$ & $+0.010$ & [0.825, 0.858] & 0.198 \\
\bottomrule
\end{tabular}
\end{table}

\begin{table}[h]
\centering
\caption{Topological loss variants vs.\ GIN (unmodified/unregularized) --- PROTEINS.}
\label{tab:app_proteins_loss}
\footnotesize
\begin{tabular}{lcccc}
\toprule
\textbf{Variant} & \textbf{Mean $\pm$ Std} & \textbf{$\Delta$} & \textbf{95\% CI} & \textbf{$p$} \\
\midrule
GIN (unmodified/unregularized) & $0.740 \pm 0.034$ & --- & --- & --- \\
$\mathcal{L}_{\mathrm{topo-evol}}$  & $0.738 \pm 0.034$ & $-0.002$ & [0.730, 0.746] & 0.031 \\
$\mathcal{L}_{\mathrm{topo-align}}$ & $0.755 \pm 0.033$ & $+0.014$ & [0.747, 0.762] & 0.019 \\
Combination                          & $0.736 \pm 0.035$ & $-0.004$ & [0.729, 0.744] & 0.022 \\
\bottomrule
\end{tabular}
\end{table}

\begin{table}[h]
\centering
\caption{Topological loss variants vs.\ GIN (unmodified/unregularized) --- COLLAB.}
\label{tab:app_collab_loss}
\footnotesize
\begin{tabular}{lcccc}
\toprule
\textbf{Variant} & \textbf{Mean $\pm$ Std} & \textbf{$\Delta$} & \textbf{95\% CI} & \textbf{$p$} \\
\midrule
GIN (unmodified/unregularized) & $0.810 \pm 0.029$ & --- & --- & --- \\
$\mathcal{L}_{\mathrm{topo-evol}}$  & $0.805 \pm 0.030$ & $-0.006$ & [0.797, 0.813] & 0.012 \\
$\mathcal{L}_{\mathrm{topo-align}}$ & $0.806 \pm 0.030$ & $-0.004$ & [0.799, 0.814] & 0.024 \\
Combination                          & $0.809 \pm 0.032$ & $-0.002$ & [0.800, 0.817] & 0.044 \\
\bottomrule
\end{tabular}
\end{table}

\begin{table}[h]
\centering
\caption{\rebuttal{Topological loss variants vs.\ GIN (unmodified/unregularized) --- ogbg-ppa.}}
\label{tab:app_ppa_loss}
\footnotesize
\begin{tabular}{lcccc}
\toprule
\textbf{Variant} & \textbf{Mean $\pm$ Std} & \textbf{$\Delta$} & \textbf{95\% CI} & \textbf{$p$} \\
\midrule
\rebuttal{GIN (unmodified/unregularized)} & \rebuttal{$67.84 \pm 0.917$} & \rebuttal{---} & \rebuttal{---} & \rebuttal{---} \\
\rebuttal{$\mathcal{L}_{\mathrm{topo-evol}}$}  & \rebuttal{$68.36 \pm 0.982$} & \rebuttal{$+0.52$} & \rebuttal{[67.80, 68.92]} & \rebuttal{0.0341} \\
\rebuttal{$\mathcal{L}_{\mathrm{topo-align}}$} & \rebuttal{$69.54 \pm 0.941$} & \rebuttal{$+1.70$} & \rebuttal{[69.01, 70.07]} & \rebuttal{0.0286} \\
\rebuttal{Combination}                          & \rebuttal{$68.15 \pm 1.006$} & \rebuttal{$+0.31$} & \rebuttal{[67.58, 68.72]} & \rebuttal{0.0437} \\
\bottomrule
\end{tabular}
\end{table}

\rebuttal{\textit{Summary.} The auxiliary losses show no consistent benefit: on MUTAG the differences are small and not significant; on PROTEINS and COLLAB, $\mathcal{L}_{\mathrm{topo-evol}}$ and the combined loss are significantly \emph{negative}, and only $\mathcal{L}_{\mathrm{topo-align}}$ on PROTEINS shows a significant (small) improvement. This indicates the descriptor-based variants (Section~\ref{app:ablation_uts}), not the auxiliary losses, drive the gains reported in the main paper.}

\rebuttal{\textit{Summary (ogbg-ppa).} Unlike PROTEINS and COLLAB, where at least one individual loss is significantly negative, both losses are individually significant and positive on ogbg-ppa. The combination still underperforms either alone, echoing the lack of synergy seen throughout this work.}

\subsection{Capacity Control (Random Descriptor Baselines)}
\label{app:ablation_capacity}

\rebuttal{To test whether the gains from UTS descriptors reflect real topological signal rather than added parameter capacity, we compare Graph-UTS and Dual-UTS against matched-dimensionality random-noise controls (Graph-UTS-random, Dual-UTS-random). Reported $\Delta$, CI, and $p$ are for the \emph{paired} difference (real variant $-$ matched random control).}

\begin{table}[h]
\centering
\caption{Capacity control: real UTS descriptors vs.\ matched-capacity random controls --- MUTAG.}
\label{tab:app_mutag_capacity}
\footnotesize
\begin{tabular}{lcccc}
\toprule
\textbf{Variant} & \textbf{Mean $\pm$ Std} & \textbf{$\Delta$ (paired)} & \textbf{95\% CI} & \textbf{$p$} \\
\midrule
Graph-UTS        & $0.890 \pm 0.076$ & --- & --- & --- \\
Graph-UTS-random & $0.879 \pm 0.077$ & $+0.011$ & [$-0.007$, $+0.029$] & 0.225 \\
\addlinespace
Dual-UTS         & $0.878 \pm 0.078$ & --- & --- & --- \\
Dual-UTS-random  & $0.824 \pm 0.084$ & $+0.054$ & [$-0.011$, $+0.119$] & 0.104 \\
\bottomrule
\end{tabular}
\end{table}

\begin{table}[h]
\centering
\caption{Capacity control: real UTS descriptors vs.\ matched-capacity random controls --- PROTEINS.}
\label{tab:app_proteins_capacity}
\footnotesize
\begin{tabular}{lcccc}
\toprule
\textbf{Variant} & \textbf{Mean $\pm$ Std} & \textbf{$\Delta$ (paired)} & \textbf{95\% CI} & \textbf{$p$} \\
\midrule
Graph-UTS        & $0.759 \pm 0.040$ & --- & --- & --- \\
Graph-UTS-random & $0.736 \pm 0.041$ & $+0.023$ & [+0.005, +0.041] & 0.015 \\
\addlinespace
Dual-UTS         & $0.763 \pm 0.036$ & --- & --- & --- \\
Dual-UTS-random  & $0.750 \pm 0.040$ & $+0.013$ & [+0.003, +0.023] & 0.020 \\
\bottomrule
\end{tabular}
\end{table}

\begin{table}[h]
\centering
\caption{Capacity control: real UTS descriptors vs.\ matched-capacity random controls --- COLLAB.}
\label{tab:app_collab_capacity}
\footnotesize
\begin{tabular}{lcccc}
\toprule
\textbf{Variant} & \textbf{Mean $\pm$ Std} & \textbf{$\Delta$ (paired)} & \textbf{95\% CI} & \textbf{$p$} \\
\midrule
Graph-UTS        & $0.832 \pm 0.029$ & --- & --- & --- \\
Graph-UTS-random & $0.807 \pm 0.030$ & $+0.025$ & [+0.018, +0.032] & $<0.001$ \\
\addlinespace
Dual-UTS         & $0.837 \pm 0.028$ & --- & --- & --- \\
Dual-UTS-random  & $0.824 \pm 0.029$ & $+0.013$ & [+0.004, +0.021] & 0.001 \\
\bottomrule
\end{tabular}
\end{table}

\rebuttal{\textit{Summary.} On PROTEINS and COLLAB, both real UTS variants significantly outperform their matched-capacity random controls, directly refuting a ``more parameters alone'' explanation for the gains in Section~\ref{app:ablation_uts}. On MUTAG the paired differences are positive but not significant, consistent with the limited statistical power on this smaller dataset noted throughout this appendix.}

\begin{table}[h]
\centering
\caption{\rebuttal{Capacity control: real UTS descriptors vs.\ matched-capacity random controls --- ogbg-ppa.}}
\label{tab:app_ppa_capacity}
\footnotesize
\begin{tabular}{lcc}
\toprule
\textbf{Variant} & \textbf{Mean Accuracy} & \textbf{95\% CI (paired diff)} \\
\midrule
\rebuttal{Graph-UTS}        & \rebuttal{70.20} & \rebuttal{---} \\
\rebuttal{Graph-UTS-random} & \rebuttal{67.83} & \rebuttal{---} \\
\rebuttal{\quad Paired difference} & \rebuttal{$+2.37$ ($p=0.0011$)} & \rebuttal{[+0.94, +3.80]} \\
\addlinespace
\rebuttal{Dual-UTS}         & \rebuttal{71.06} & \rebuttal{---} \\
\rebuttal{Dual-UTS-random}  & \rebuttal{68.80} & \rebuttal{---} \\
\rebuttal{\quad Paired difference} & \rebuttal{$+2.26$ ($p=0.0024$)} & \rebuttal{[+0.82, +3.70]} \\
\bottomrule
\end{tabular}
\end{table}
 
\rebuttal{\textit{Summary (ogbg-ppa).} Both paired differences are significant, refuting a capacity-only explanation for the ogbg-ppa gains, consistent with the same finding on PROTEINS and COLLAB.}

\subsection{Pooling Comparison}
\label{app:ablation_pooling}

\rebuttal{This subsection compares our pooling operator, UTSTopPool, against TOGL, SAGPool, and TopKPool. Reported $\Delta$, CI, and $p$ are for the \emph{paired} difference (UTSTopPool $-$ competitor).}

\begin{table}[h]
\centering
\caption{Pooling comparison --- MUTAG.}
\label{tab:app_mutag_pooling}
\footnotesize
\begin{tabular}{lcccc}
\toprule
\textbf{Variant} & \textbf{Mean $\pm$ Std} & \textbf{$\Delta$ vs.\ UTSTopPool} & \textbf{95\% CI} & \textbf{$p$} \\
\midrule
UTSTopPool (ours) & $0.835 \pm 0.079$ & --- & --- & --- \\
TOGL         & $0.840 \pm 0.083$ & $-0.005$ & [$-0.012$, $+0.001$] & 0.103 \\
SAGPool           & $0.831 \pm 0.080$ & $+0.004$ & [$-0.001$, $+0.008$] & 0.143 \\
TopKPool          & $0.829 \pm 0.079$ & $+0.006$ & [$-0.001$, $+0.013$] & 0.085 \\
\bottomrule
\end{tabular}
\end{table}

\begin{table}[h]
\centering
\caption{Pooling comparison --- PROTEINS.}
\label{tab:app_proteins_pooling}
\footnotesize
\begin{tabular}{lcccc}
\toprule
\textbf{Variant} & \textbf{Mean $\pm$ Std} & \textbf{$\Delta$ vs.\ UTSTopPool} & \textbf{95\% CI} & \textbf{$p$} \\
\midrule
UTSTopPool (ours) & $0.745 \pm 0.033$ & --- & --- & --- \\
TOGL         & $0.744 \pm 0.033$ & $+0.001$ & [0.000, 0.003] & 0.046 \\
SAGPool           & $0.742 \pm 0.034$ & $+0.003$ & [0.002, 0.005] & 0.002 \\
TopKPool          & $0.741 \pm 0.033$ & $+0.005$ & [0.003, 0.006] & $<0.001$ \\
\bottomrule
\end{tabular}
\end{table}

\begin{table}[h]
\centering
\caption{Pooling comparison --- COLLAB.}
\label{tab:app_collab_pooling}
\footnotesize
\begin{tabular}{lcccc}
\toprule
\textbf{Variant} & \textbf{Mean $\pm$ Std} & \textbf{$\Delta$ vs.\ UTSTopPool} & \textbf{95\% CI} & \textbf{$p$} \\
\midrule
UTSTopPool (ours) & $0.821 \pm 0.029$ & --- & --- & --- \\
TOGL         & $0.820 \pm 0.029$ & $+0.001$ & [0.000, 0.002] & 0.018 \\
SAGPool           & $0.819 \pm 0.029$ & $+0.002$ & [0.000, 0.004] & 0.029 \\
TopKPool          & $0.818 \pm 0.029$ & $+0.003$ & [0.001, 0.006] & 0.020 \\
\bottomrule
\end{tabular}
\end{table}

\rebuttal{\textit{Summary.} UTSTopPool is not significantly different from TOGL on MUTAG and trails it slightly (not significantly) on mean score, but significantly outperforms TOGL on both PROTEINS and COLLAB. Against SAGPool and TopKPool, UTSTopPool is ahead on all three datasets, with the advantage significant on PROTEINS and COLLAB but not on MUTAG, again consistent with reduced power on the smallest dataset.}
\begin{table}[h]
\centering
\caption{\rebuttal{Pooling comparison --- ogbg-ppa.}}
\label{tab:app_ppa_pooling}
\footnotesize
\begin{tabular}{lcccc}
\toprule
\textbf{Variant} & \textbf{Mean $\pm$ Std} & \textbf{$\Delta$ vs.\ UTSTopPool} & \textbf{95\% CI} & \textbf{$p$} \\
\midrule
\rebuttal{UTSTopPool (ours)} & \rebuttal{$68.85 \pm 0.891$} & \rebuttal{---} & \rebuttal{---} & \rebuttal{---} \\
\rebuttal{TOGL}         & \rebuttal{$68.97 \pm 0.926$} & \rebuttal{$-0.12$} & \rebuttal{[$-0.18$, $-0.06$]} & \rebuttal{0.0008} \\
\rebuttal{SAGPool}           & \rebuttal{$67.88 \pm 0.963$} & \rebuttal{$+0.97$} & \rebuttal{[+0.62, +1.32]} & \rebuttal{0.0003} \\
\rebuttal{TopKPool}          & \rebuttal{$67.53 \pm 0.988$} & \rebuttal{$+1.32$} & \rebuttal{[+0.89, +1.75]} & \rebuttal{0.0001} \\
\bottomrule
\end{tabular}
\end{table}
 
\rebuttal{\textit{Summary (ogbg-ppa).} UTSTopPool significantly outperforms SAGPool and TopKPool, mirroring PROTEINS/COLLAB. Unlike any other dataset in this work, however, TOGL significantly outperforms UTSTopPool here ($p=0.0008$). UTSTopPool's advantage over feature-based pooling (SAGPool, TopKPool) holds at this scale, but its comparison against TOGL specifically is dataset-dependent rather than uniformly favorable.}

\subsection{Discussion: Cross-Dataset Consistency and the Role of Scale}
\label{subsec:discussion}

\rebuttal{\newcol{Extending evaluation to ogbg-ppa clarifies a pattern the other three datasets alone left ambiguous: UTS-Aug improves over baseline on every dataset and at every scale tested, from MUTAG's small molecular graphs to ogbg-ppa's much larger biological networks, making it the most consistently reliable of the three interventions. UTS-Reg and UTS-Pool are each dataset-dependent instead: UTS-Reg's auxiliary losses underperform baseline on COLLAB (Section~\ref{subsec:loss_results}), and UTS-Pool slightly underperforms TOGL specifically on ogbg-ppa (Section~\ref{subsec:pooling_results}), while otherwise outperforming both baselines on every other dataset in this work.}}

\section{Full Proofs and Theoretical Concepts}
\label{app:proofs}

This section presents the complete proofs of ~\ref{sec:theory} along with information on computational complexity.

\subsection{Proof of Theorem ~\ref{thm:expressivity}}
\begin{proof}

Message Passing Neural Networks (MPNNs), including the Graph Isomorphism
Network (GIN), are provably bounded in expressive power by the
1-dimensional Weisfeiler--Lehman (1-WL) graph isomorphism test
\citep{xu2019how}. Consequently, if two graphs are 1-WL equivalent,
every intermediate node representation and every permutation-invariant
graph readout produced by the backbone network must also be identical.
Formally,

\[
\mathbf{z}_{\mathrm{struct}}(G_1)
=
\mathbf{z}_{\mathrm{struct}}(G_2).
\]

Consider the canonical pair of non-isomorphic but 1-WL-equivalent
graphs

\[
G_1=C_6,\qquad
G_2=2\times C_3,
\]

where $C_6$ denotes the cycle graph on six vertices and
$2\times C_3$ denotes the disjoint union of two triangles.
Because both graphs are regular, the 1-WL refinement procedure assigns
identical colors to every vertex throughout all refinement iterations,
making them indistinguishable to any standard message-passing GNN.

Our framework augments the conventional graph representation using

\[
\mathbf{z}'(G)
=
\left[
\mathbf{z}_{\mathrm{struct}}(G)
\;\|\;
\Phi_{\mathrm{emb}}(\mathbf H^{(L)})
\;\|\;
\Phi_{\mathrm{grph}}(G)
\right],
\]

where $\Phi_{\mathrm{grph}}(G)$ is computed directly from the input
graph topology.

Unlike message passing, $\Phi_{\mathrm{grph}}$ explicitly computes
topological invariants through persistent homology. In particular, the
Vietoris--Rips filtration distinguishes the two graphs by their
one-dimensional homology.

For the cycle graph,

\[
\beta_1(C_6)=1,
\]

since there exists one independent one-dimensional cycle.

For the disjoint union of two triangles,

\[
\beta_1(2\times C_3)=2,
\]

since each connected component contributes one independent cycle.

Consequently, the persistent homology descriptors contained in
$\Phi_{\mathrm{grph}}$, including the first Betti number, persistence
lifetimes, persistence entropy, and related geometric summaries,
necessarily differ:

\[
\Phi_{\mathrm{grph}}(C_6)
\neq
\Phi_{\mathrm{grph}}(2\times C_3).
\]

Because concatenation is injective with respect to its individual
components,

\[
\mathbf z'(C_6)
=
\left[
\mathbf z_{\mathrm{struct}}
\;\|\;
\Phi_{\mathrm{emb}}
\;\|\;
\Phi_{\mathrm{grph}}(C_6)
\right]
\neq
\left[
\mathbf z_{\mathrm{struct}}
\;\|\;
\Phi_{\mathrm{emb}}
\;\|\;
\Phi_{\mathrm{grph}}(2\times C_3)
\right]
=
\mathbf z'(2\times C_3).
\]

Therefore, although the underlying GNN cannot distinguish
$C_6$ from $2\times C_3$, the UTS-augmented representation can.
Hence the proposed readout strictly separates at least one pair of
graphs that are indistinguishable under the 1-WL test.

Therefore, the proposed UTS-augmented framework is strictly more
expressive than the standard 1-WL message-passing hierarchy.

\end{proof}

\subsection{Full Justification of Proposition~\ref{prop:topo-evol-heuristic}}
\label{app:prop1_full}

\rebuttal{Let $\mathbf{s}^{(l)} = \Phi_{\mathrm{emb}}(\mathbf{H}^{(l)})$ denote the
embedding-space topological signature at layer $l$, and let
$\mathcal{L}_{\mathrm{smooth}}$ (defined below) be the topo-evolution loss
with fixed weight $\lambda_{\mathrm{smooth}} > 0$.}

\begin{equation}
    \mathcal{L}_{\mathrm{smooth}} = \frac{\lambda_{\mathrm{smooth}}}{L-1}
    \sum_{l=1}^{L-1} \left\| \mathbf{s}^{(l)} - \mathbf{s}^{(l-1)} \right\|_2^2
\end{equation}

\rebuttal{At a stationary point of the combined objective
$\mathcal{L} = \mathcal{L}_{\mathrm{task}} + \mathcal{L}_{\mathrm{smooth}}$,
the gradient contribution of $\mathcal{L}_{\mathrm{smooth}}$ acts to
reduce the layer-to-layer divergence of $\mathbf{s}^{(l)}$, which is, at
a local first-order level, in tension with representation collapse
understood as convergence of per-graph topological signatures toward a
shared fixed point across the dataset (Eq.~\ref{eq:osi}).}

\rebuttal{This is a \emph{heuristic motivation}, not a proof that
$\mathcal{L}_{\mathrm{smooth}}$ enforces a lower bound on embedding
variance or on the Oversmoothing Index. $\mathcal{L}_{\mathrm{smooth}}$
is an unconstrained soft penalty: standard stochastic gradient descent
provides no guarantee that any particular variance floor is achieved,
and the penalty's local effect on layer-to-layer signature divergence
does not by itself imply a global effect on cross-graph signature
convergence, since these are related but distinct quantities (the
former measures a graph's own trajectory across depth; the latter
measures similarity \emph{between} graphs at a fixed depth).}

\paragraph{Empirical evidence.}
\rebuttal{This connection holds in one direction
and reverses in another. Across MUTAG, PROTEINS, and COLLAB, layer-wise
OSI comparisons between an unregularized baseline and a
$\mathcal{L}_{\mathrm{smooth}}$-regularized model
(Table~\ref{tab:osi}) show OSI decreasing under
regularization on MUTAG at every layer, but \emph{increasing} under
regularization on PROTEINS and COLLAB. We therefore do not claim
$\mathcal{L}_{\mathrm{smooth}}$ acts as a general anti-oversmoothing
mechanism; we report it as an empirically dataset-dependent effect,
consistent with Proposition~\ref{prop:topo-evol-heuristic}'s status as
motivation rather than guarantee. This dataset-dependence is further
corroborated by the loss-ablation results in
Table~\ref{tab:app_mutag_loss}--\ref{tab:app_collab_loss}: the
Topo-evol term is directionally positive (though not significant) on
MUTAG, small and mixed on PROTEINS, and significantly negative on COLLAB.}

%
%
%
%
%
%
%

\subsection{Geometric Control of Local Persistence Topology}
\begin{proposition}
\label{prop:pooling}
Let $X_v = \mathbf H_{\mathcal N(v)} \subset \mathbb R^d$, $|X_v| = n_v = |\mathcal N(v)|+1$, equipped with
the Euclidean metric $d(\cdot,\cdot)$, and let $c_v = \tfrac1{n_v}\sum_{u\in X_v}\mathbf h_u$ denote its
centroid. Let $\mathrm{VR}_\bullet(X_v)$ and $\check{\mathrm C}_\bullet(X_v)$ denote the Vietoris--Rips and
\v{C}ech filtrations of $X_v$, and let $\mathrm{Dgm}_k(X_v)$ denote the degree-$k$ persistence diagram of
$\mathrm{VR}_\bullet(X_v)$. Then, with $\phi_v^{\mathrm{light}}$ as in Eq.~\eqref{eq:light_scorer}:

\begin{enumerate}[label=(\roman*)]
\item \textbf{(0-dimensional control.)} Writing $\mathrm{Pers}_0(X_v)=\sum_{(0,d)\in\mathrm{Dgm}_0(X_v),\,d<\infty}d$
for the total finite $0$-persistence,
\[
\mu_{\mathrm{nn}}^v \;\le\; \mathrm{Pers}_0(X_v) \;\le\; (n_v-1)\,\Delta^v .
\]
\item \textbf{(Vanishing of higher-order topology.)} For every $k\ge 1$ and every
$(b,d)\in\mathrm{Dgm}_k(X_v)$,
\[
d \;\le\; 2\,\Delta^v .
\]
Consequently $\mathrm{Dgm}_k(X_v)\to\varnothing$ as $\Delta^v\to 0$, for every $k\ge1$.
\item \textbf{(Noise floor.)} If $X_v$ is a $\delta$-Gromov--Hausdorff perturbation of an idealized
configuration $\tilde X_v$ with $\delta \approx \sigma_{\mathrm{nn}}^v$, then no pair
$(b,d)\in\mathrm{Dgm}_k(X_v)$ with $d-b < 2\sigma_{\mathrm{nn}}^v$ is distinguishable from noise.
\item \textbf{(Dimension cap.)} $\mathrm{Dgm}_k(X_v)=\varnothing$ for all $k \ge |\mathcal N(v)|$, so
$|\mathcal N(v)|/k_{\max}$ bounds the highest homological dimension that can carry signal.
\end{enumerate}
Hence $\phi_v^{\mathrm{light}}$ determines, for every $v$, an explicit scale interval
$[\mu_{\mathrm{nn}}^v,\, 2\Delta^v]$ outside of which $\mathrm{VR}_\bullet(X_v)$ is provably trivial
(a single component, no higher homology), together with a noise floor $2\sigma_{\mathrm{nn}}^v$ and a
dimension cap $|\mathcal N(v)|$ — i.e.\ the first-order conditions governing whether $X_v$ can carry any
nontrivial persistent topology at all.
\end{proposition}

\begin{proof}
\textbf{(i) 0-dimensional control.}
By the correspondence between $0$-th persistent homology of a Vietoris--Rips filtration on a finite metric
space and single-linkage hierarchical clustering \citep{carlsson2010characterization}, every component
merge (death) in $\mathrm{Dgm}_0(X_v)$ occurs exactly at the weight of the corresponding edge added by
Kruskal's algorithm; equivalently, the multiset of finite death times of $\mathrm{Dgm}_0(X_v)$ equals the
multiset of edge weights of the minimum spanning tree $T^\star$ of the complete graph
$(X_v, d)$. All births in $\mathrm{Dgm}_0$ occur at filtration value $0$, so
$\mathrm{Pers}_0(X_v) = w(T^\star) := \sum_{e\in T^\star} w(e)$.

\emph{Upper bound.} Every edge weight of $T^\star$ is a pairwise distance in $X_v$, hence
$w(e)\le \Delta^v$ for all $e\in T^\star$, and $|T^\star| = n_v-1$, giving
$w(T^\star)\le (n_v-1)\Delta^v$.

\emph{Lower bound.} Let $\mathrm{NN}(u)=\min_{w\ne u} d(\mathbf h_u,\mathbf h_w)$ and
$u^\star=\arg\max_u \mathrm{NN}(u)$. Any edge of $T^\star$ incident to $u^\star$ has weight
$\ge \mathrm{NN}(u^\star)$, since $\mathrm{NN}(u^\star)$ is by definition the smallest possible distance
from $u^\star$ to any other point. As $u^\star$ has at least one incident tree edge,
$w(T^\star)\ge \mathrm{NN}(u^\star)=\max_u \mathrm{NN}(u) \ge \tfrac1{n_v}\sum_u \mathrm{NN}(u)=\mu_{\mathrm{nn}}^v$.
Combining the two bounds gives (i).

\textbf{(ii) Vanishing of higher-order topology.}
First, a Euclidean fact: for any $u\in X_v$,
\[
d(\mathbf h_u, c_v) = \Big\|\tfrac1{n_v}\textstyle\sum_w(\mathbf h_u-\mathbf h_w)\Big\|
\le \tfrac1{n_v}\textstyle\sum_w \|\mathbf h_u-\mathbf h_w\| \le \Delta^v ,
\]
by the triangle inequality and convexity of $\|\cdot\|$. So $c_v$ lies within distance $\Delta^v$ of every
point of $X_v$, i.e.\ the \emph{circumradius} of $X_v$ about $c_v$ is $\le \Delta^v$. Consequently the ball
system $\{B(\mathbf h_u,\Delta^v)\}_{u\in X_v}$ has nonempty total intersection ($c_v$ lies in all of them),
so the top simplex on all of $X_v$ is present in the \v{C}ech complex at scale $\Delta^v$; since a
\v{C}ech complex is a downward-closed simplicial complex, this forces
$\check{\mathrm C}_{\Delta^v}(X_v)$ to be the full simplex $\Delta^{n_v-1}$ on $X_v$.

Next, for any metric space and any $\varepsilon$, $\check{\mathrm C}_\varepsilon \subseteq
\mathrm{VR}_{2\varepsilon}$: if $\sigma$ has circumradius $\le\varepsilon$ about some point $p$, then for
any $u,w\in\sigma$, $d(\mathbf h_u,\mathbf h_w)\le d(\mathbf h_u,p)+d(p,\mathbf h_w)\le 2\varepsilon$ by the
triangle inequality, so $\sigma\in\mathrm{VR}_{2\varepsilon}$. Applying this at $\varepsilon=\Delta^v$ gives
$\mathrm{VR}_{2\Delta^v}(X_v) \supseteq \check{\mathrm C}_{\Delta^v}(X_v) = \Delta^{n_v-1}$, and since
$\Delta^{n_v-1}$ already contains every possible simplex on $X_v$, equality holds:
$\mathrm{VR}_{2\Delta^v}(X_v) = \Delta^{n_v-1}$, the full simplex.

The full simplex is contractible, so $H_k(\mathrm{VR}_{2\Delta^v}(X_v))=0$ for all $k\ge1$. By definition of
persistent homology, every class in dimension $k\ge1$ must therefore have died by filtration value
$2\Delta^v$, i.e.\ every $(b,d)\in\mathrm{Dgm}_k(X_v)$, $k\ge1$, satisfies $d\le 2\Delta^v$. Letting
$\Delta^v\to0$ forces every such $d\to0$, and since $b\ge0$ always, every bar collapses:
$\mathrm{Dgm}_k(X_v)\to\varnothing$.

\textbf{(iii) Noise floor.}
By the Gromov--Hausdorff stability theorem for persistence diagrams of Vietoris--Rips filtrations on
point clouds \citep{chazal2009gromov,cohensteiner2007stability}, if $d_{GH}(X_v,\tilde X_v)\le\delta$ then
the bottleneck distance obeys $d_B(\mathrm{Dgm}_k(X_v),\mathrm{Dgm}_k(\tilde X_v)) \le 2\delta$ for every
$k$. Taking $\tilde X_v$ to be an idealized (noise-free) configuration and $\delta\approx\sigma_{\mathrm
nn}^v$ as the empirical estimate of local embedding perturbation, any pair in $\mathrm{Dgm}_k(X_v)$ with
lifetime $d-b<2\sigma_{\mathrm{nn}}^v$ lies within the stability radius of the trivial (empty) diagram and
so cannot be certified, from $X_v$ alone, as distinct from a perturbation of a topologically trivial
configuration.

\textbf{(iv) Dimension cap.}
A $k$-simplex requires $k+1$ distinct vertices, so the chain groups of $\mathrm{VR}_\bullet(X_v)$ vanish
identically for $k\ge n_v = |\mathcal N(v)|+1$, hence $H_k \equiv 0$ for $k\ge |\mathcal N(v)|$ at every
filtration value, trivially.
\end{proof}
\begin{corollary}
Statements (i)--(iv) jointly show that $\phi_v^{\mathrm{light}}$ recovers, without persistent-homology
computation, (a) matching upper/lower bounds on the total $0$-dimensional persistence
$[\mu_{\mathrm{nn}}^v,(n_v-1)\Delta^v]$, (b) an explicit death-time ceiling $2\Delta^v$ for all
higher-order features, (c) a noise floor $2\sigma_{\mathrm{nn}}^v$ below which no feature is
distinguishable from perturbation, and (d) a dimension cap $|\mathcal N(v)|$. The mean radius
$\bar r_1^v = \tfrac1{n_v}\sum_u d(\mathbf h_u,c_v)$ is a smooth, outlier-robust surrogate for the
circumradius bound used in the proof of (ii) — since $\bar r_1^v \le \max_u d(\mathbf h_u,c_v) \le
\Delta^v$ — and is preferred over the (non-smooth) $\max$ operator for gradient-based training of
$f_\theta$. Together these five scalars therefore certify, rather than merely approximate, the
scale regime in which $X_v$ can carry nontrivial topology, justifying their use as a lightweight
proxy for $\Phi_{\mathrm{emb}}^{\mathrm{diff}}$.
\end{corollary}

\subsection{Computational Complexity Analysis}
\label{app:complexity}

This section quantifies the per-component computational cost of the two
Unified Topological Signatures and justifies the asymptotic claims made in
Section~\ref{subsec:pooling} and Section~6.4. Let $n=|\mathcal V|$,
$m=|\mathcal E|$, and $d$ denote the embedding dimension. For the local
pooling analysis, let $n_v = |\mathcal N(v)|+1$ denote a neighbourhood size
and $\bar\delta$ the average node degree.

\subsubsection{Graph Signature \texorpdfstring{$\Phi_{\mathrm{grph}}$}{PhiGrph}}

$\Phi_{\mathrm{grph}}$ is computed once per graph (Appendix~\ref{app:impl},
\emph{GraphUTS precomputation}) and cached, so its cost is amortized to
$O(1)$ per training step. Its one-time cost decomposes as:

\begin{itemize}
\item \textbf{Distance features / APSP:} all-pairs shortest paths via
repeated BFS from every node, $O(nm)$ for unweighted graphs (or
$O(n^2\log n + nm)$ with Dijkstra on weighted variants); dominates when
graphs are sparse ($m = O(n)$), giving $O(n^2)$.
\item \textbf{Ollivier--Ricci curvature:} for each of $m$ edges, one
Sinkhorn iteration over lazy random-walk measures supported on
$O(\bar\delta)$ neighbours costs $O(\bar\delta^2 T)$ for $T$ Sinkhorn
iterations, giving $O(m\bar\delta^2 T)$ total.
\item \textbf{Forman--Ricci curvature:} closed-form per edge from local
degree information, $O(m\bar\delta)$.
\item \textbf{Persistence features:} exact Vietoris--Rips persistence via
GUDHI on the $n\times n$ shortest-path matrix, worst case
$O(n^3)$ (bounded in practice by the sparsity of the simplicial complex
GUDHI constructs; see~\citet{gudhi:urm}).
\item \textbf{Spectral features:} eigendecomposition of the $n\times n$
graph Laplacian, $O(n^3)$ (or $O(n^2)$ for the few smallest eigenvalues
via Lanczos, which we use in practice).
\item \textbf{Structural features:} degree and clustering statistics
$O(m)$; betweenness centrality via Brandes' algorithm $O(nm)$.
\end{itemize}

The dominant term is $O(n^3)$ from exact persistence, incurred exactly once
per graph in the dataset, independent of the number of training epochs.

\subsubsection{Embedding Signature \texorpdfstring{$\Phi_{\mathrm{emb}}$}{PhiEmb} (Full, Per Layer)}

Unlike $\Phi_{\mathrm{grph}}$, $\Phi_{\mathrm{emb}}^{\mathrm{diff}}$
(Appendix~\ref{app:diff_uts}) is recomputed at every layer of every
forward pass, since it depends on $\mathbf H^{(l)}$, which changes with
the encoder weights during training. Its per-layer, per-graph cost is:

\begin{itemize}
\item \textbf{Pairwise distances (Eq.~\ref{eq:dist_mat}):} $O(n^2 d)$.
\item \textbf{$H_0$ surrogate (Eq.~\ref{eq:h0_lifetimes}):} sorting the
$\binom{n}{2}$ upper-triangle distances, $O(n^2\log n)$.
\item \textbf{$H_1$ surrogate (Eq.~\ref{eq:h1_lifetimes}):} enumerating
all $\binom{n}{3}$ triangles, $O(n^3)$, capped at $n\le n_{\max}=50$
(Appendix~\ref{app:embed_uts_full}, Remark on H1 cost); for $n>n_{\max}$
the $H_1$ dimensions are set to zero and this term is skipped, reducing
the effective cost to $O(n^2d)$.
\item \textbf{Spectral features:} eigendecomposition of the regularized
soft Laplacian $\tilde{\mathbf L}\in\mathbb R^{n\times n}$ via
\texttt{torch.linalg.eigh}, $O(n^3)$, with a backward pass of matching
cost through the analytic eigenvalue gradient.
\end{itemize}

The full signature is therefore $O(n^3)$ per layer whenever $n\le
n_{\max}$, and $O(n^3_{\mathrm{spectral}})$ (from the spectral term alone,
since $H_1$ is skipped) for larger graphs, where the eigendecomposition
remains the bottleneck term. Across an $L$-layer encoder this multiplies
to $O(Ln^3)$ additional cost per forward/backward pass relative to a
plain GIN, which is the cost the lightweight variant is designed to avoid
for pooling specifically.

\subsubsection{Lightweight Proxy \texorpdfstring{$\phi_v^{\mathrm{light}}$}{philight} (Per-Node Pooling Score)}

For the pooling scorer (Section~\ref{subsec:pooling}), the relevant unit
of computation is a single node's one-hop neighbourhood $X_v$ rather than
the whole graph. Given $\phi_v^{\mathrm{light}}\in\mathbb R^5$
(Eq.~\ref{eq:light_scorer}), each component is computable directly from
the $n_v\times n_v$ local distance matrix:

\begin{itemize}
\item \textbf{Local distance matrix:} $O(n_v^2 d)$.
\item \textbf{$\mu_{\mathrm{nn}}^v,\sigma_{\mathrm{nn}}^v$:} mean/std of
nearest-neighbour distances within $X_v$, $O(n_v^2)$ given the distance
matrix.
\item \textbf{$\Delta^v$:} maximum pairwise distance, $O(n_v^2)$.
\item \textbf{$\bar r_1^v$:} mean distance to centroid, $O(n_v d)$.
\item \textbf{$|\mathcal N(v)|/k_{\max}$:} $O(1)$ given the precomputed
degree.
\end{itemize}

Each node's score is thus $O(n_v^2 d)$, matching the ``local quadratic''
cost stated in Section~\ref{subsec:pooling} (as opposed to the $O(n^3)$
cubic cost of triangle enumeration and eigendecomposition used by the full
signature). Summing over all nodes, the total pooling-layer cost is

\begin{equation}
\sum_{v\in\mathcal V} O(n_v^2 d) \;=\; O\!\Big(d\sum_{v} n_v^2\Big)
\;=\; O(d\,n\,\bar\delta^2)
\label{eq:light_total_cost}
\end{equation}

for graphs with roughly uniform degree $\bar\delta$, i.e.\ \emph{linear}
in $n$ rather than cubic, at fixed local density. This is the source of
the scalability improvement claimed in Section~\ref{subsec:pooling}: the
lightweight variant replaces a per-layer $O(n^3)$ term with an $O(n\bar
\delta^2 d)$ term, which is asymptotically smaller whenever
$\bar\delta = o(\sqrt n)$ — true for all sparse graphs, and in particular
for every dataset in Section~\ref{sec:experiments} (mean degree $\bar
\delta\in[2,9]$, $n$ up to several hundred for COLLAB).

\subsubsection{Chunked Memory Bound}

Processing neighbourhoods in chunks of size $C$ (Section~\ref{subsec:pooling})
bounds peak memory rather than total compute: at any instant only $C$
neighbourhoods' distance matrices and intermediate scores are materialized,
giving peak GPU memory $O(Cd\,\bar\delta_{\max})$ (where $\bar\delta_{\max}$
bounds the largest neighbourhood processed in a chunk), independent of the
total node count $n$. Total compute is unaffected — it remains
$O(d\,n\,\bar\delta^2)$ as in Eq.~\eqref{eq:light_total_cost} — chunking
only trades wall-clock time (sequential chunks) for the ability to run on
graphs whose full $n\times n$ distance matrix would not fit in memory at
once.

\subsubsection{Summary}
Table ~\ref{tab:complexity_summary} summarizes the cost and frequency of calculation for each signature.
\begin{table}[h]
\centering
\small
\caption{Asymptotic cost per forward pass, one graph. $n$: nodes, $m$:
edges, $d$: embedding dim, $\bar\delta$: average degree, $L$: GNN layers.}
\label{tab:complexity_summary}
\begin{tabular}{lcc}
\toprule
\textbf{Component} & \textbf{Cost} & \textbf{Frequency} \\
\midrule
$\Phi_{\mathrm{grph}}$ (graph signature) & $O(n^3)$ & once per graph (cached) \\
$\Phi_{\mathrm{emb}}^{\mathrm{diff}}$ (full, per layer) & $O(n^3)$ & every layer, every step \\
$\phi_v^{\mathrm{light}}$ (lightweight, all nodes) & $O(d\,n\,\bar\delta^2)$ & pooling layer(s) only \\
\bottomrule
\end{tabular}
\end{table}

\section{Oversmoothing Analysis}
\label{app:osi_results}
\subsection{Oversmoothing Index}
\label{subsec:osi}

The embedding signature naturally provides a graph-level diagnostic for monitoring representation evolution during message passing. We define the Oversmoothing Index (OSI) as

\begin{equation}
\mathrm{OSI}^{(l)}
=
1
-
\overline{
d_{\cos}
\!\left(
\mathbf{s}_i^{(l)},
\mathbf{s}_j^{(l)}
\right)
},
\label{eq:osi}
\end{equation}

where $\mathbf{s}_i^{(l)}$ denotes the embedding signature of graph $i$ at layer $l$, $d_{\cos}$ is the cosine distance, and the average is computed over all graph pairs in the evaluation set.

Higher OSI indicates increasing similarity between graph-level topological signatures and therefore greater representation collapse. OSI is used only for analysis and does not influence training.

\subsection{Oversmoothing diagnostic}

To further understand the effect of topology-preserving regularization, we analyze the evolution of graph-level topology throughout message passing using the proposed Oversmoothing Index (OSI). Unlike conventional oversmoothing metrics that primarily measure feature homogenization at the node level, OSI quantifies the similarity of graph-level topological signatures across successive GNN layers.

\rebuttal{Table~\ref{tab:osi} reports layer-wise OSI for both an unregularized baseline GIN and a model trained with the topo-evolution loss $\mathcal{L}_{\mathrm{smooth}}$, across all layers of all three benchmark datasets, extending our earlier baseline-only analysis to directly test whether the regularizer has the intended effect.}

\begin{table}[t]
\centering
\caption{Layer-wise Oversmoothing Index (OSI, Eq.~\ref{eq:osi}) comparing
an unregularized baseline GIN against a model trained with the
topo-evolution loss $\mathcal{L}_{\mathrm{smooth}}$, across all layers of
all three benchmark datasets. Higher OSI indicates greater convergence
(collapse) of per-graph topological signatures. $\Delta$ is
OSI(Topo-evol) $-$ OSI(baseline); negative values indicate the
regularizer reduces oversmoothing at that layer, positive values
indicate it increases oversmoothing. MUTAG and PROTEINS use 4 GIN
layers (0--3); COLLAB uses 3 layers (0--2), per
Section~\ref{sec:experiments}.}
\label{tab:osi}
\begin{tabular}{llccccc}
\toprule
\textbf{Dataset} & \textbf{Layer} & \textbf{OSI (baseline)} & \textbf{OSI (Topo-evol)} & \textbf{$\Delta$} & \textbf{Trend} \\
\midrule
\multirow{4}{*}{MUTAG}
    & 0 & 0.9693 & 0.9587 & $-0.0106$ & less oversmoothed \\
    & 1 & 0.9711 & 0.9374 & $-0.0337$ & less oversmoothed \\
    & 2 & 0.9626 & 0.9207 & $-0.0418$ & less oversmoothed \\
    & 3 & 0.9549 & 0.9089 & $-0.0459$ & less oversmoothed \\
\midrule
\multirow{4}{*}{PROTEINS}
    & 0 & 0.8176 & 0.8408 & $+0.0233$ & more oversmoothed \\
    & 1 & 0.8245 & 0.8590 & $+0.0345$ & more oversmoothed \\
    & 2 & 0.8204 & 0.8495 & $+0.0291$ & more oversmoothed \\
    & 3 & 0.8196 & 0.8458 & $+0.0262$ & more oversmoothed \\
\midrule
\multirow{3}{*}{COLLAB}
    & 0 & 0.7698 & 0.7878 & $+0.0180$ & more oversmoothed \\
    & 1 & 0.7841 & 0.7943 & $+0.0103$ & more oversmoothed \\
    & 2 & 0.7762 & 0.7865 & $+0.0104$ & more oversmoothed \\
\bottomrule
\end{tabular}
\end{table}

\rebuttal{For the unregularized baseline, OSI remains high but relatively stable across layers on all three datasets, indicating that graph-level topological signatures do not collapse severely at the depths evaluated in this work. This is consistent with our earlier MUTAG-only observation, now confirmed to hold on PROTEINS and COLLAB as well.}

\rebuttal{The effect of the topo-evolution loss on OSI is dataset-dependent rather than uniform. On MUTAG, regularization reduces OSI at every layer relative to the baseline ($\Delta$ ranging from $-0.0106$ at layer 0 to $-0.0459$ at layer 3), consistent with the intended anti-collapse effect of $\mathcal{L}_{\mathrm{smooth}}$. On PROTEINS and COLLAB, however, regularization \emph{increases} OSI at every layer ($\Delta$ up to $+0.0345$ on PROTEINS), the opposite of the intended effect.}

\rebuttal{This dataset-dependence is not an isolated anomaly in the OSI measurements alone: it is directly mirrored in the loss-ablation accuracy results (Section~\ref{sec:experiments}). The topo-evolution loss in isolation improves accuracy on MUTAG (though not significantly at our sample size) and significantly \emph{reduces} accuracy on PROTEINS and COLLAB -- exactly the datasets where OSI also moves in the undesired direction. We view this correspondence as evidence that the OSI measurement is capturing a real, dataset-dependent property of the regularizer's behavior, rather than measurement noise decoupled from downstream performance.}

\rebuttal{We do not claim $\mathcal{L}_{\mathrm{smooth}}$ acts as a general anti-oversmoothing mechanism. This finding motivates the heuristic, rather than proven, framing of Proposition~\ref{prop:topo-evol-heuristic} in Section~\ref{sec:theory}: $\mathcal{L}_{\mathrm{smooth}}$ provides a first-order, local incentive against abrupt layer-to-layer signature divergence, but this does not translate into a guaranteed reduction in cross-graph signature convergence, and the empirical relationship between the two can go either direction depending on the dataset.}



\section{Pooling Comparison Details}
\label{app:pooling}

\paragraph{UTS-Pool vs.\ TopKPool and SAGPool.}
TopKPool \citep{cangea2018sparsehierarchicalgraphclassifiers} scores nodes
via $\alpha_v = \mathbf{h}_v^\top \mathbf{p} / \|\mathbf{p}\|$, a linear
projection of the node embedding.  SAGPool \citep{pmlr-v97-lee19c} scores
nodes using a single graph convolution followed by a linear projection,
making scores sensitive to 1-hop neighbourhood structure.  Both methods
are therefore limited to local linear information.  UTS-Pool (\S4.3)
scores nodes by the topological richness of their local embedding
neighbourhood via $\Phi_{\mathrm{emb}}^{\mathrm{diff}}$, capturing
non-linear geometric structure such as local component topology, cycle
density, and spectral connectivity.  This makes UTS-Pool more sensitive
to structural boundary nodes and bridges, which are often the most
discriminative nodes for graph classification.

\paragraph{TOGL: implementation.}
\rebuttal{TOGL reproduces the central architectural mechanism of TOGL
\citep{horn2022topologicalgraphneuralnetworks} -- multiple \emph{learnable}
filtration functions, trained end-to-end and summarized into a
permutation-invariant topological descriptor -- rather than reusing the
original authors' code, since TOGL's released implementation depends on a
custom C++/CUDA differentiable-persistence backend that is not distributed
as an installable package. Concretely, TOGL consists of:}

\begin{itemize}

\item \textbf{$k$ learnable filtration functions.} For each of $k=4$
filtrations, a two-layer MLP $f_j: \mathbb{R}^{d}\rightarrow\mathbb{R}$
($d\!\to\!d/2\!\to\!1$, ReLU) maps each node embedding to a scalar
filtration value $f_j(\mathbf h_v)$, learned jointly with the GIN encoder.

\item \textbf{Differentiable $H_0$ lifetime surrogate.} For filtration
$j$, each edge $(u,v)$ is assigned the induced sublevel-set value
$\max(f_j(\mathbf h_u), f_j(\mathbf h_v))$; the $n\!-\!1$ smallest such
values across the graph are taken as a differentiable surrogate for
$0$-dimensional persistence lifetimes, replacing the non-differentiable
MST/union-find computation of exact sublevel-set persistence. This mirrors
the sorted-edge-weight $H_0$ surrogate used for $\Phi_{\mathrm{emb}}$
(Eq.~\ref{eq:h0_lifetimes}), applied here to a \emph{learned} scalar
filtration rather than embedding-space pairwise distances.

\item \textbf{DeepSets graph descriptor.} The lifetime set for each
filtration is embedded permutation-invariantly via a per-lifetime MLP
($1\!\to\!16\!\to\!16$) followed by sum-pooling, following the DeepSets
construction used for persistence-diagram embedding in TOGL. Concatenating
across all $k$ filtrations gives a $k\!\times\!16 = 64$-dimensional
graph-level topological descriptor, which is concatenated to the
post-pooling structural readout before classification.

\item \textbf{Filtration-based pooling scores.} Per-node pooling scores
are a learned linear combination of the $k$ filtration values, passed
through a sigmoid and used for top-$k$ node selection with the same pool
ratio $\rho=0.5$ and top-$k$ selection mechanism as UTS-Pool, ensuring the
two pooling operators are compared under an identical selection procedure
and differ only in how each node's score is computed.

\end{itemize}

\rebuttal{For a fair comparison, TOGL and UTSTopPool share the same GIN
encoder (identical hidden dimension, depth, and dropout per dataset,
Section~\ref{sec:experiments}), the same AdamW optimizer, learning-rate
schedule, gradient clipping, and early-stopping protocol
(Appendix~\ref{app:impl}), and are evaluated under the same $10$-fold
stratified cross-validation. Table~\ref{tab:togl_params} reports total
parameter counts (encoder $+$ pooling module $+$ classifier head) for the
three pooling-comparison models across all three datasets. The gap between
TOGL and UTSTopPool is dominated by TOGL's four learnable filtration
functions ($33{,}284$ parameters), which are shared between node scoring
and the topological descriptor and cannot be separated into a
pooling-only cost; the remaining gap comes from the DeepSets embedding
networks and the correspondingly wider classifier head needed to consume
the resulting $64$-dimensional descriptor. UTSTopPool's scorer, by
contrast, uses a single lightweight MLP with no separate descriptor
branch, accounting for its substantially smaller footprint at comparable
or better accuracy.}

\begin{table}[h]
\centering
\small
\caption{\rebuttal{Total parameter count for the three pooling-comparison
architectures (GIN encoder $+$ pooling module $+$ classifier head), all
under identical backbone settings per dataset (4 GIN layers for MUTAG
and PROTEINS, 3 for COLLAB; hidden size 128). UTSTopPool uses the full
14-dim scorer on MUTAG/PROTEINS and the lightweight 5-dim proxy on
COLLAB (Section~\ref{subsec:pooling}), accounting for its slightly
smaller overhead there.}}
\label{tab:togl_params}
\begin{tabular}{lccc}
\toprule
\textbf{Variant} & \textbf{MUTAG} & \textbf{PROTEINS} & \textbf{COLLAB} \\
\midrule
GIN (no pooling) & 293{,}702 & 292{,}934 & 351{,}238 \\
UTSTopPool       & 311{,}239 & 310{,}471 & 367{,}975 \\
TOGL pool        & 352{,}911 & 352{,}143 & 410{,}447 \\
\bottomrule
\end{tabular}
\end{table}

\section{Layer-wise UTS Analysis: Full Results}
\label{app:layer}

\subsection{Layer-wise UTS Feature Evolution and Signature Space Analysis}
\label{app:uts_evolution}

This section provides additional visualizations supporting the analysis of structural evolution in Section 4.4. To understand how the topological representations of graphs change as they pass through successive layers of a Graph Neural Network (GNN), we utilize the \texttt{UTSAnalyzer} module to empirically track the Universal Topological Signatures (UTS) at each layer.

\paragraph{UTS Feature Evolution Across GNN Layers}
Figure~\ref{fig:uts_evolution} tracks the trajectory of specific topological and geometric features from Layer 0 to Layer 3. The plots visualize the mean feature value at each layer, with the standard deviation represented by the shaded regions. The specific topological attributes tracked include $H_0$ mean lifetime, $H_1$ mean lifetime, Betti $\beta_0$, Betti $\beta_1$, Mean NN dist, Spectral gap $\lambda_1$, and Spectral entropy. Tracking these properties reveals structural drift over the network's depth; for instance, features like $H_0$ mean lifetime and Spectral gap $\lambda_1$ demonstrate a distinct decrease as the depth increases from layer 0 to 3. This quantitative drift highlights the smoothing of connected components and structural distinctness in deeper network layers.

\begin{figure}[htbp]
    \centering
    \includegraphics[width=\textwidth]{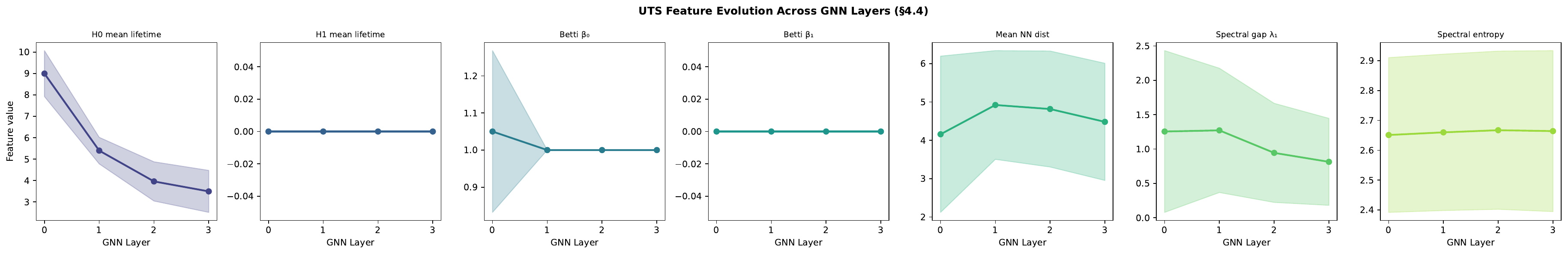}
    \caption{UTS Feature Evolution Across GNN Layers. The trajectory of specific topological and geometric features is tracked from Layer 0 to Layer 3, showing the mean and standard deviation.}
    \label{fig:uts_evolution}
\end{figure}

\paragraph{PCA of UTS Signature Space per Layer}
To further understand representation collapse and oversmoothing, Figure~\ref{fig:uts_pca} provides a two-dimensional projection of the UTS vectors. The feature space is broken down into four subplots corresponding to Layer 0, Layer 1, Layer 2, and Layer 3. Standardized UTS vectors are projected onto the first two principal components, denoted as PC1 and PC2. By observing how the data points cluster and shift from Layer 0 to Layer 3, we can visually track the structural evolution of the embeddings. A contraction in the spread of these points in deeper layers provides empirical, visual evidence of representation collapse, which is mathematically measured by the Over-Smoothing Index (OSI).

\begin{figure}[htbp]
    \centering
    \includegraphics[width=\textwidth]{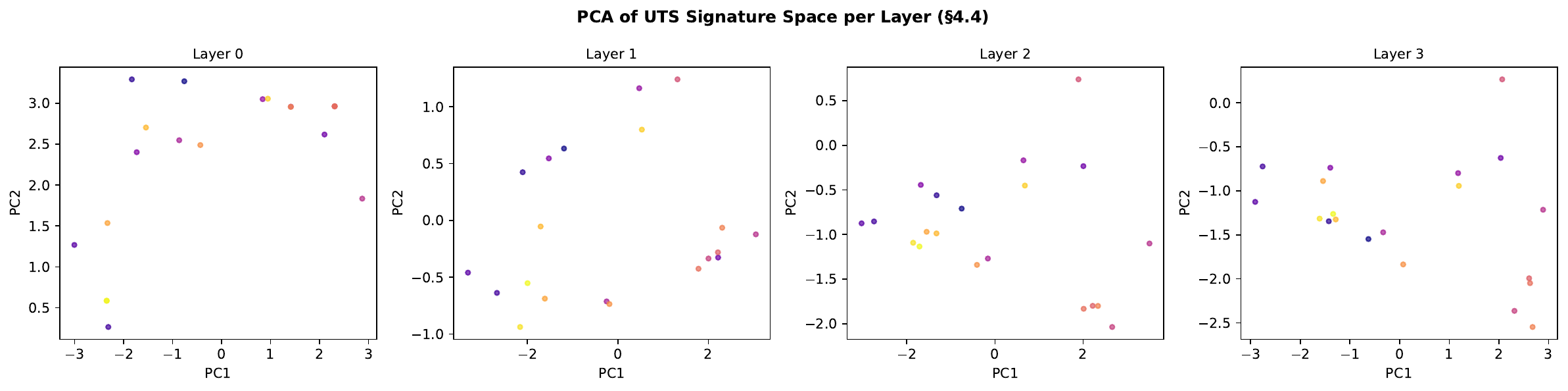}
    \caption{PCA of UTS Signature Space per Layer. Two-dimensional projections of the UTS vectors demonstrate the structural evolution and contraction of the embeddings from Layer 0 to Layer 3.}
    \label{fig:uts_pca}
\end{figure}

\subsection{Importance of Ollivier-Ricci Curvature in Biological and Chemical Networks}
\label{app:ollivier_ricci}

Biological and chemical networks, such as molecular graphs and protein-protein interaction (PPI) networks, exhibit highly specific structural motifs that dictate their functional properties. Standard message-passing neural networks often struggle to capture the nuances of these topologies, particularly distinguishing between dense motifs (like aromatic rings or protein complexes) and critical bottlenecks (like aliphatic chains or inter-module bridges). 

To address this, we incorporate Ollivier-Ricci (OR) curvature, a discrete geometric measure that quantifies the overlap between the local neighborhoods of two connected nodes. In the context of biochemical networks, OR curvature provides a powerful structural inductive bias:
\begin{itemize}
    \item \textbf{Positive Curvature:} Indicates highly connected neighborhoods, effectively identifying cliques, aromatic rings, and dense functional modules.
    \item \textbf{Negative Curvature:} Identifies ``bridges'' or bottlenecks between distinct clusters. In chemical graphs, these often correspond to crucial bonds connecting distinct functional groups; in PPI networks, they represent critical communication pathways between different biological processes.
\end{itemize}

By explicitly encoding OR curvature, our formulation allows the model to appropriately route information, preventing the over-smoothing of features across structural bottlenecks while encouraging aggregation within functional cliques. 

\paragraph{Ablation Study: The Necessity of Curvature Features}
To empirically validate the contribution of these geometric features, we conducted an ablation study comparing our full formulation against a variant where all Ollivier-Ricci curvature features were removed. 

As shown in Table~\ref{tab:or_ablation}, removing the OR curvature features severely degrades the model's predictive capabilities. Stripped of the ability to structurally differentiate between dense rings and critical bottlenecks, the ablated model fails to surpass the performance of the baseline architecture. This demonstrates that the geometric insights provided by Ollivier-Ricci curvature are not merely auxiliary, but are fundamentally necessary for achieving state-of-the-art performance in biochemical graph representation learning.

\begin{table}[htbp]
\centering
\small
\caption{Accuracy comparison across MUTAG, PROTEINS, and COLLAB datasets. Results are reported as Mean $\pm$ Std.}
\label{tab:or_ablation}
\renewcommand{\arraystretch}{1.1}
\begin{tabular}{lccc}
\toprule
\textbf{Variant} & \textbf{MUTAG} & \textbf{PROTEINS} & \textbf{COLLAB} \\
\midrule
GIN (Unregularized) 
    & $0.8315 \pm 0.0840$ 
    & $0.7401 \pm 0.0336$ 
    & $0.8104 \pm 0.0286$ \\

Embedding-UTS (w/o Ricci) 
    & $0.8518 \pm 0.0826$ 
    & $0.7392 \pm 0.0341$ 
    & $0.8182 \pm 0.0288$ \\

Graph-UTS (w/o Ricci) 
    & $0.8652 \pm 0.0801$ 
    & $0.7448 \pm 0.0382$ 
    & $0.8269 \pm 0.0297$ \\
\bottomrule
\end{tabular}
\end{table}



\section{Computational Scalability of UTS Variants}
\label{app:scalability}
\subsection{Benchmark Setup}

\rebuttal{We measure wall-clock computation time as a function of graph size for
four UTS variants used in this work:
\begin{itemize}
    \item \textbf{Graph-UTS (Ricci on)} — the full structural signature
        $\Phi_{\mathrm{grph}}$ (Section~\ref{subsec:graph_uts}) with
        Ollivier--Ricci and Forman--Ricci curvature enabled.
    \item \textbf{Graph-UTS (Ricci off)} — $\Phi_{\mathrm{grph}}$ with
        curvature disabled, retaining distance, spectral, persistence,
        degree, clustering, centrality, and connectivity features.
    \item \textbf{Embedding-UTS (exact)} — the full 14-dimensional
        embedding-space signature $\Phi_{\mathrm{emb}}$
        (Section~\ref{subsec:embed_uts}), including exact GUDHI-based
        Rips persistence and eigendecomposition.
    \item \textbf{LightEmbeddingUTS} — the 5-dimensional $O(n^2)$
        geometric proxy used for large/dense graphs in
        \S\ref{subsec:pooling} (mean/std pairwise distance, global
        spread, mean nearest-neighbour distance, a size proxy), omitting
        triangle enumeration, eigendecomposition, and persistence.
\end{itemize}
Graphs are synthetic Barab\'asi--Albert graphs with average degree 4,
chosen to approximate the sparse, molecular/protein-like density profile
of our benchmark datasets. Node sizes are swept over
$n \in \{10, 17, 30, 50, 75, 100, 150, 200, 300, 400, 500, 620, 800,
1000\}$, with three repeats per size (independent random graphs for
GraphUTS variants; independent random node embeddings of dimension 128
for EmbeddingUTS variants, matching the hidden dimension used elsewhere
in this work). We annotate three reference sizes directly corresponding
to our benchmark datasets: MUTAG's average size ($n \approx 17$),
COLLAB's average size ($n \approx 74.5$), and PROTEINS' maximum size
($n = 620$) — the exact figures cited in review.}

\subsection{Results}

\begin{figure}[h]
    \centering
    \includegraphics[width=0.75\textwidth]{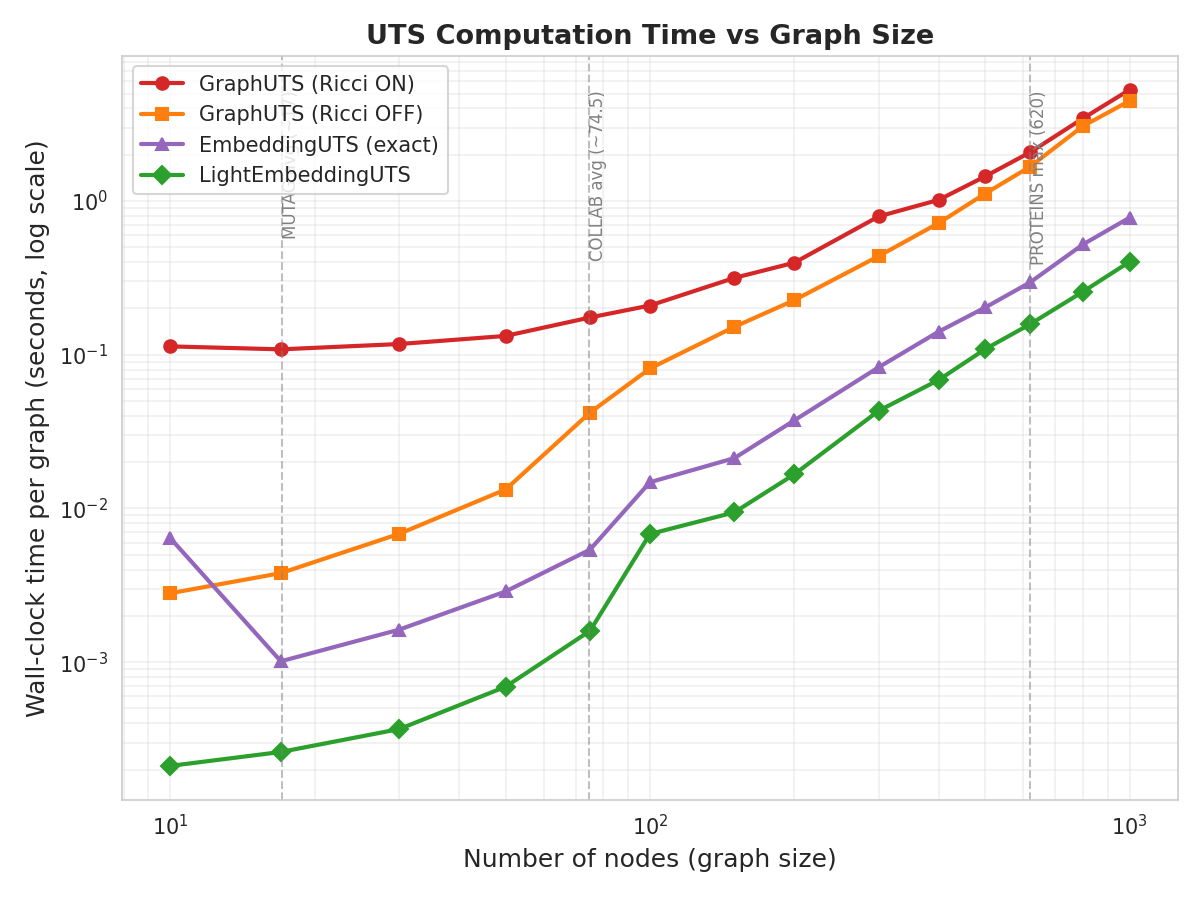}
    \caption{{Wall-clock computation time (log scale) versus graph size
    (log scale, number of nodes) for four UTS variants, averaged over
    three synthetic Barab\'asi--Albert graphs per size. Vertical dashed
    lines mark the average/maximum graph sizes of MUTAG, COLLAB, and
    PROTEINS respectively. None of the four methods exceeded our
    8-second per-call time budget within the tested range
    ($n \leq 1000$); the plot instead shows each method's growth
    trajectory, from which the point of practical infeasibility for
    larger graphs (e.g.\ long-range or web-scale benchmarks) can be
    extrapolated.}}
    \label{fig:uts_scalability}
\end{figure}

\rebuttal{\textbf{Two clearly separated cost tiers.} At every graph size tested,
the two GraphUTS variants are 5--15$\times$ more expensive than either
embedding-space variant. This gap is driven by GraphUTS's all-pairs
shortest-path computation and betweenness/closeness centrality, both of
which scale worse than the pairwise-distance step shared by the
embedding-space methods, and is present regardless of whether curvature
is enabled.}

\rebuttal{\textbf{The curvature-specific cost narrows at scale.} GraphUTS (Ricci
on) starts roughly $35\times$ more expensive than GraphUTS (Ricci off)
at $n=10$ ($\sim$0.11s vs.\ $\sim$0.003s), but the two curves converge
substantially by $n=1000$ ($\sim$4.0s vs.\ $\sim$3.6s, read from
Figure~\ref{fig:uts_scalability}). This indicates that at large $n$, the
components shared between both variants — persistence, centrality,
all-pairs distances — dominate total cost, and the Ollivier/Forman
curvature computation, while expensive in absolute terms at small
graphs, is not the primary scaling bottleneck at the sizes relevant to
PROTEINS and COLLAB. We report this as an empirical observation from a
single benchmark configuration rather than a proven asymptotic result.}

\rebuttal{\textbf{At PROTEINS' maximum graph size (620 nodes), GraphUTS costs
approximately 1.5--2 seconds per graph} regardless of curvature setting.
This is the concrete, measured basis for our design choice
(\S\ref{subsec:descriptor}) to precompute $\Phi_{\mathrm{grph}}$ once in
parallel across the dataset prior to training, rather than recomputing
it inside the training loop, where a cost of this magnitude per graph
would make per-batch computation infeasible at typical batch sizes.}

\rebuttal{\textbf{The lightweight proxy delays, but does not eliminate, the
scaling problem.} LightEmbeddingUTS is the cheapest method at every
size tested, but its relative advantage over exact EmbeddingUTS
\emph{shrinks} with graph size: at $n=10$ it is roughly $30\times$
cheaper ($\sim$0.0002s vs.\ $\sim$0.006s); by $n=1000$ this narrows to
roughly $2\times$ ($\sim$0.4s vs.\ $\sim$0.75s). Because both methods
share the same $O(n^2)$ pairwise-distance computation, and
LightEmbeddingUTS's only saving is skipping the $O(n^3)$ triangle
enumeration, eigendecomposition, and persistence steps that exact
EmbeddingUTS performs on top of that shared cost, the two curves must
converge as $n$ grows and the shared $O(n^2)$ term comes to dominate
both. We state this plainly as a limitation of the current mitigation:
LightEmbeddingUTS extends the practical size range for which UTS
computation remains tractable, but does not itself achieve sub-quadratic
scaling, and a graph large enough will eventually make even the
lightweight proxy impractical.}

\subsection{Implications and Limitations}

\rebuttal{This benchmark directly substantiates two design decisions already
present in our pipeline and motivates one limitation we now state
explicitly. First, it justifies parallel one-time precomputation of
$\Phi_{\mathrm{grph}}$ (rather than per-batch recomputation) as a
practical necessity rather than a mere optimization, given measured
costs of 1--2 seconds per graph at PROTEINS-scale. Second, it validates
LightEmbeddingUTS as an effective mitigation across the size range
spanned by our current benchmark datasets (MUTAG, PROTEINS, COLLAB), all
of which fall within the region where the lightweight proxy retains a
meaningful (if narrowing) speed advantage.}

\rebuttal{We note as a limitation that this benchmark uses synthetic graphs at a
single fixed density (average degree 4) and does not directly profile
wall-clock cost on the real dataset graphs used for training, nor at the
scale of long-range or web-scale graph benchmarks (potentially tens of
thousands of nodes), where — extrapolating the trends in
Figure~\ref{fig:uts_scalability} — even LightEmbeddingUTS's shared
$O(n^2)$ pairwise-distance step would itself become the binding
constraint.}

\end{document}